\documentclass{article}

\PassOptionsToPackage{numbers, compress}{natbib}

\usepackage[final]{neurips_2019}

\usepackage[utf8]{inputenc} % allow utf-8 input
\usepackage[T1]{fontenc}    % use 8-bit T1 fonts
\usepackage{url}            % simple URL typesetting
\usepackage{booktabs}       % professional-quality tables
\usepackage{amsfonts}       % blackboard math symbols
\usepackage{nicefrac}       % compact symbols for 1/2, etc.
\usepackage{microtype}      % microtypography
\usepackage{graphicx}
\usepackage{times}
\usepackage{epsfig}
\usepackage{color}
\usepackage{amsmath}
\usepackage{amssymb}

\usepackage{wrapfig}
\usepackage{subcaption}
\usepackage{multirow}
\usepackage{adjustbox}
\usepackage{graphicx}

\usepackage{diagbox}
\usepackage{float}
\restylefloat{table}
\usepackage[pdfencoding=auto, psdextra]{hyperref}
\usepackage{amsmath,amsfonts,bm,amsthm}
\usepackage{amssymb}
\usepackage{xcolor}

\newcommand{\norm}[1]{\left\Vert#1\right\Vert}

\newcommand{\abs}[1]{\left\vert#1\right\vert}

\newcommand{\set}[1]{\left\{#1\right\}}

\newcommand{\brac}[1]{\left [#1\right ]}

\newcommand{\Real}{\mathbb R}

\newcommand{\too}{\rightarrow}

\newcommand{\dist}{\textrm{d}} %distance function

\makeatletter
\newtheorem*{rep@theorem}{\rep@title}
\newcommand{\newreptheorem}[2]{%
	\newenvironment{rep#1}[1]{%
		\def\rep@title{#2 \ref{##1}}%
		\begin{rep@theorem}}%
		{\end{rep@theorem}}}
\makeatother

\newtheorem{theorem}{Theorem}
\newreptheorem{theorem}{Theorem}
\newtheorem{lemma}{Lemma}
\newreptheorem{lemma}{Lemma}

\makeatletter
\newcommand{\subalign}[1]{%
	\vcenter{%
		\Let@ \restore@math@cr \default@tag
		\baselineskip\fontdimen10 \scriptfont\tw@
		\advance\baselineskip\fontdimen12 \scriptfont\tw@
		\lineskip\thr@@\fontdimen8 \scriptfont\thr@@
		\lineskiplimit\lineskip
		\ialign{\hfil$\m@th\scriptstyle##$&$\m@th\scriptstyle{}##$\crcr
			#1\crcr
		}%
	}
}
\makeatletter

\def \etal{{et al}.}
\newcommand{\eg}{{e.g.}}
\newcommand{\ie}{{i.e.}}

\def\eqref#1{equation~\ref{#1}}
\def\Eqref#1{Equation~\ref{#1}}
\def\1{\bm{1}}

\DeclareMathAlphabet{\mathsfit}{\encodingdefault}{\sfdefault}{m}{sl}
\SetMathAlphabet{\mathsfit}{bold}{\encodingdefault}{\sfdefault}{bx}{n}

\def\gL{{\mathcal{L}}}
\def\gM{{\mathcal{M}}}

\def\gO{{\mathcal{O}}}

\def\gS{{\mathcal{S}}}
\def\gT{{\mathcal{T}}}

\def\gX{{\mathcal{X}}}

\DeclareMathOperator*{\argmax}{arg\,max}

\title{Controlling Neural Level Sets}

\author{%
	Matan Atzmon,
	Niv Haim,  
	Lior Yariv,
	Ofer Israelov,
	Haggai Maron,
	Yaron Lipman \\
	Weizmann Institute of Science\\
	Rehovot, Israel \\
}

\begin{document}
	
	\maketitle
	
	\begin{abstract}
		
		The level sets of neural networks represent fundamental properties such as decision boundaries of classifiers and are used to model non-linear manifold data such as curves and surfaces. Thus, methods for controlling the neural level sets could find many applications in machine learning.
		
		In this paper we present a simple and scalable approach to directly control level sets of a deep neural network. Our method consists of two parts: (i) sampling of the neural level sets, and (ii) relating the samples' positions to the network parameters. The latter is achieved by a \emph{sample network} that is constructed by adding a single fixed linear layer to the original network. In turn, the sample network can be used to incorporate the level set samples into a loss function of interest.
		
		We have tested our method on three different learning tasks: improving generalization to unseen data, training networks robust to adversarial attacks, and curve and surface reconstruction from point clouds. For surface reconstruction, we produce high fidelity surfaces directly from raw 3D point clouds. When training small to medium networks to be robust to adversarial attacks we obtain robust accuracy comparable to state-of-the-art methods. 
		
		%For Notably, we increase robust accuracy to the level of standard classification accuracy in off-the-shelf networks, improving it by 2\% in MNIST and 27\% in CIFAR10 compared to state-of-the-art methods. 
		
	\end{abstract}

	\section{Introduction}
	
	%\yl{add points and results from rebuttal....XXX}
	
	%motivation
	The level sets of a Deep Neural Network (DNN) are known to capture important characteristics and properties of the network. A popular example is when the network $F(x;\theta):\Real^d\times \Real^m \too \Real^l$ represents a classifier, $\theta$ are its learnable parameters, $f_i(x;\theta)$ are its logits (the outputs of the final linear layer), and the level set 
	\begin{equation}\label{e:decision_boundary}
	\gS(\theta)=\set{x\in\Real^d \ \Big\vert \ f_j-\max_{i\ne j} \set{f_i} = 0 }
	\end{equation}
	represents the decision boundary of the $j$-th class. Another recent example is modeling a manifold (\eg, a curve or a surface in $\Real^3$) using a level set of a neural network (e.g., \cite{park2019deepsdf}). That is, 
	% for a network $F(x;\theta):\Real^d\times \Real^m \too \Real^l$, 
	\begin{equation}\label{e:manifold}
	\gS(\theta)=\set{x\in \Real^d \ \Big\vert \ F = 0 }
	\end{equation}
	represents (generically) a manifold of dimension $d-l$ in $\Real^d$.

	%goal and approach
	The goal of this work is to provide practical means to directly control and manipulate \emph{neural level sets} $\gS(\theta)$, as exemplified in Equations \ref{e:decision_boundary}, \ref{e:manifold}. The main challenge is how to incorporate $\gS(\theta)$ in a differentiable loss function. Our key observation is that given a sample $p\in \gS(\theta)$, its position can be associated to the network parameters: $p=p(\theta)$, in a \emph{differentiable} and \emph{scalable} way. In fact, $p(\theta)$ is itself a neural network that is obtained by an addition of a \emph{single} linear layer to $F(x;\theta)$; we call these networks \emph{sample networks}. Sample networks, together with an efficient mechanism for sampling the level set, $\{p_j(\theta)\}\subset\gS(\theta)$, can be  incorporated in general loss functions as a proxy for the level set $\gS(\theta)$.
	
	\begin{figure}
		\begin{tabular}{cccc}
			\includegraphics[width=0.22\columnwidth]{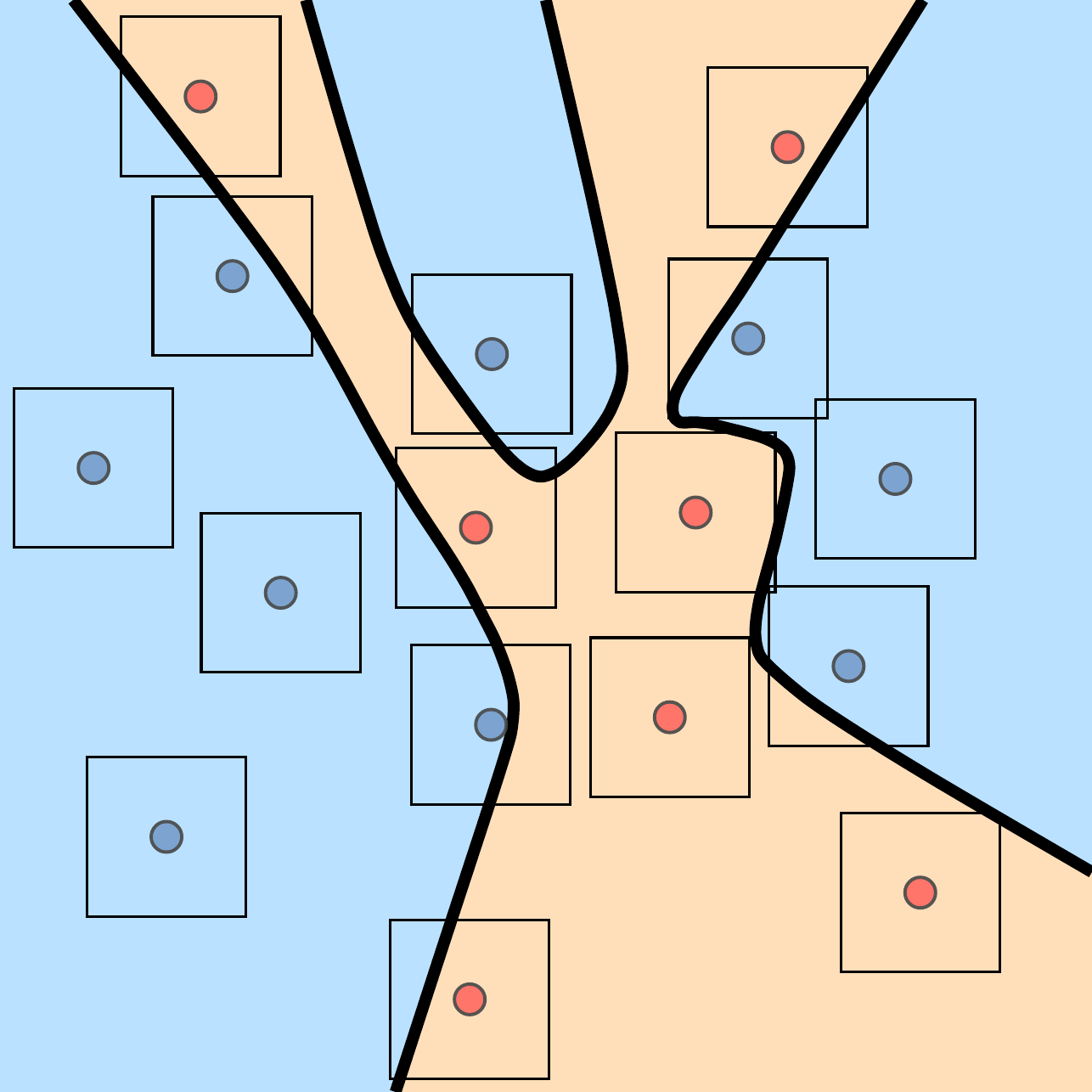} & 
			\includegraphics[width=0.22\columnwidth]{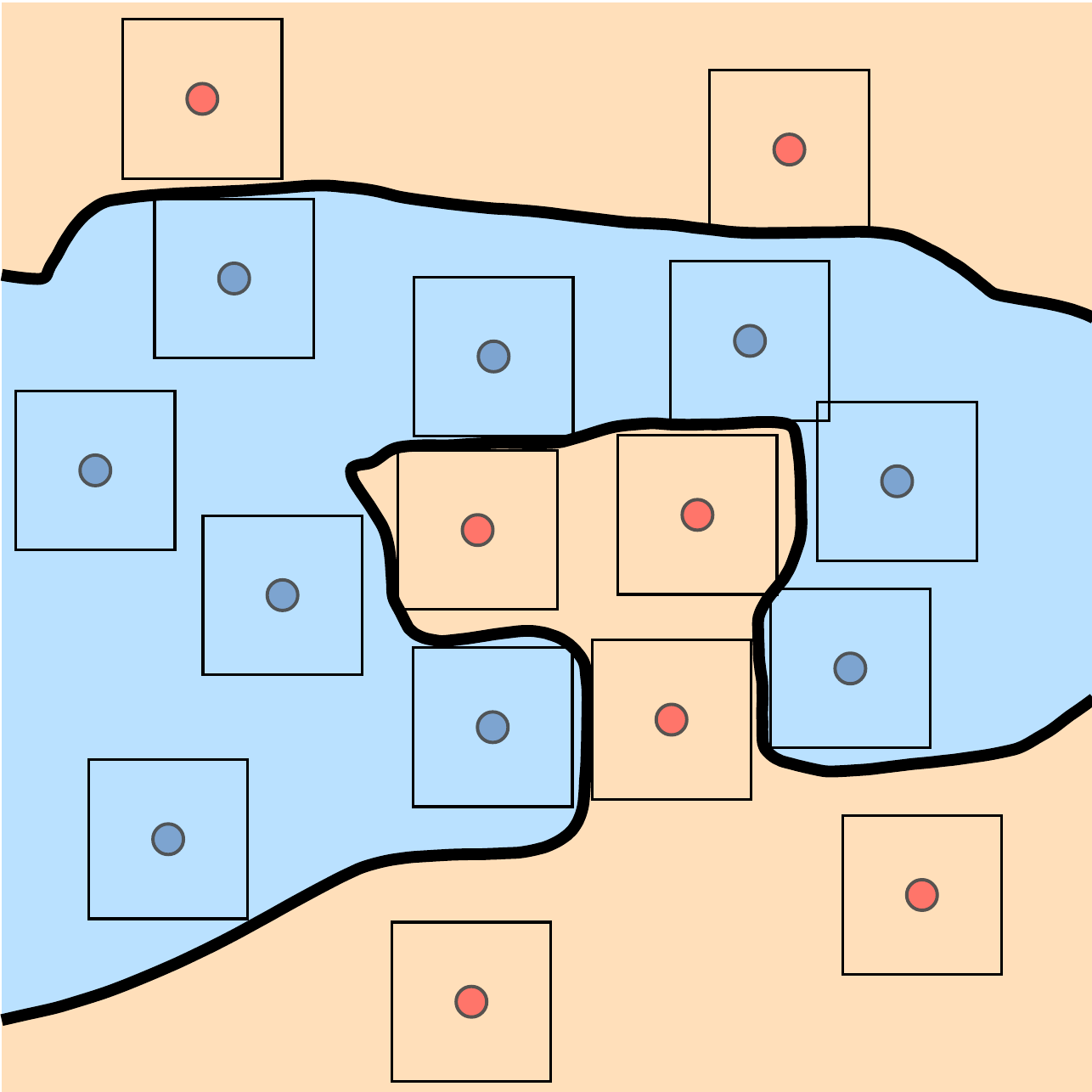} & 
			\includegraphics[width=0.22\columnwidth]{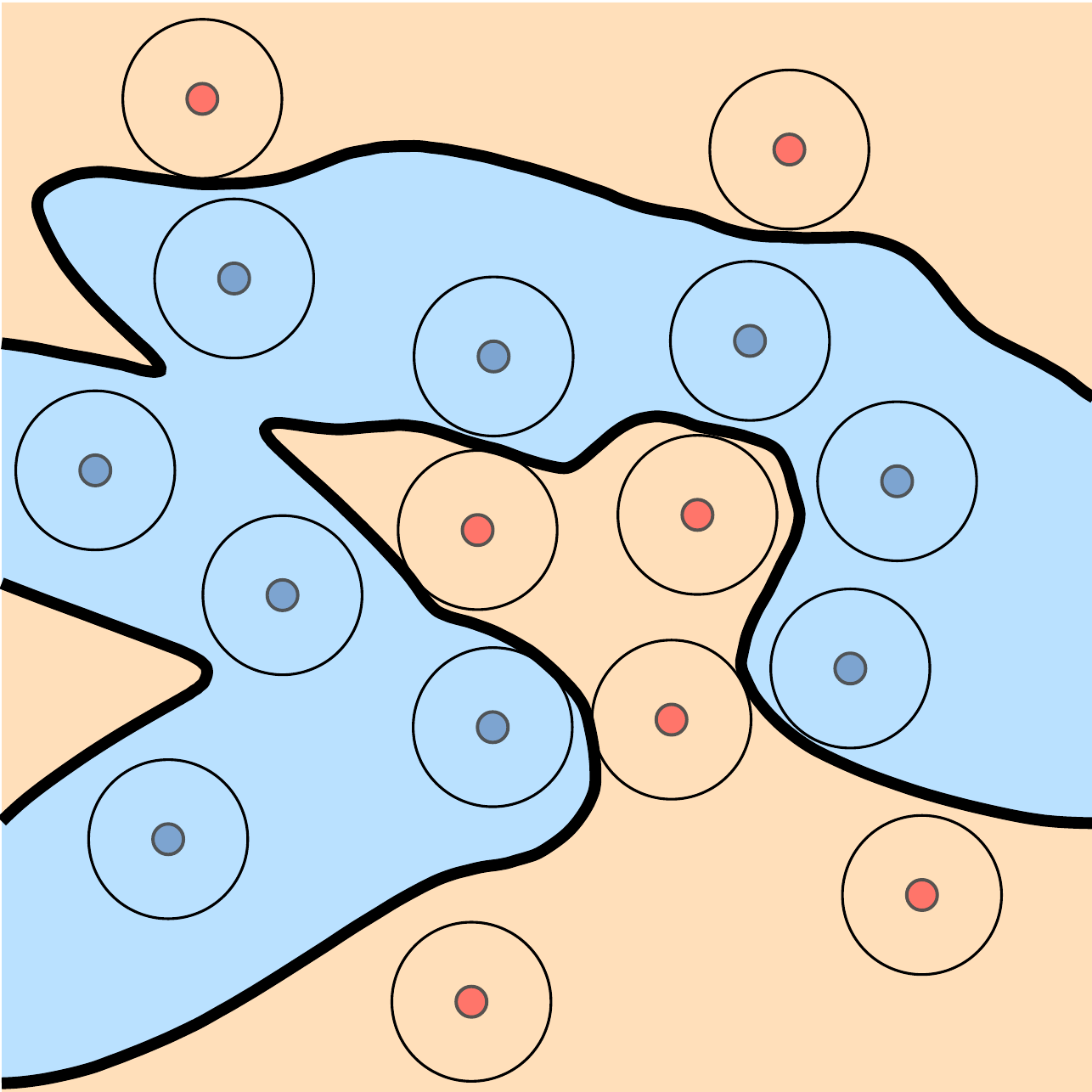} & 
			\includegraphics[width=0.22\columnwidth]{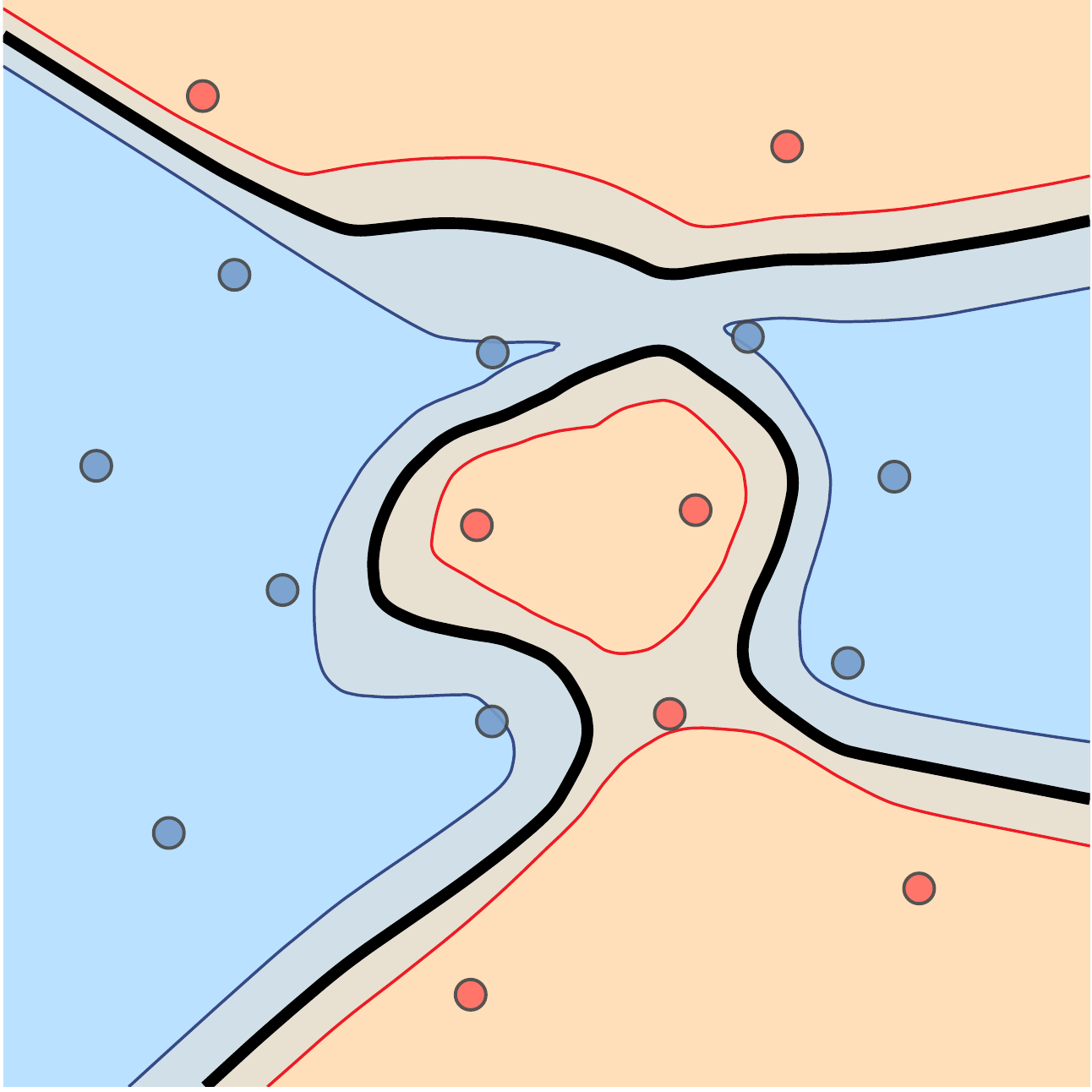} \\
			(a) & (b) & (c) & (d) 
		\end{tabular}
		\caption{Our method applied to binary classification in 2D. Blue and red dots represent positive and negative examples respectively. (a) standard cross-entropy loss baseline; (b) using our method to move the decision boundary at least $\varepsilon$ away from the training set in $L_\infty$ norm; (c) same for $L_2$ norm; (d) a geometrical adaptation of SVM soft-margin loss to neural networks using our method, the $+1,-1$ level sets are marked in light red and blue, respectively. Note that in (b),(c),(d) we achieve decision boundaries that seem to better explain the training examples compared to (a). }
		\label{fig:teaser}
		\vspace{-10pt}
	\end{figure}
	
	%applications
	We apply our method of controlling the neural level sets to two seemingly different learning tasks: controlling decision boundaries of classifiers (\Eqref{e:decision_boundary}) and reconstructing curves and surfaces from point cloud data (\Eqref{e:manifold}). 
	
	At first, we use our method to improve the generalization and adversarial robustness in classification tasks. In these tasks, the distance between the training examples and the decision boundary is an important quantity called the \emph{margin}. Margins are traditionally desired to be as large as possible to improve generalization \cite{cortes1995support,elsayed2018large} and adversarial robustness \cite{elsayed2018large}. Usually, margins are only controlled indirectly by loss functions that measure network output values of training examples (\eg, cross entropy loss). Recently, researchers have been working on optimizations with more direct control of the margin using linearization techniques \cite{hein2017formal,matyasko2017margin,elsayed2018large}, regularization \cite{sokolic2017robust}, output-layer margin \cite{sun2015large}, or using margin gradients \cite{ding2018max}. We suggest controlling the margin by sampling the decision boundary, constructing the sample network, and measuring distances between samples and training examples directly. By applying this technique to train medium to small-size networks against adversarial perturbations we achieved comparable robust accuracy to state-of-the-art methods.   
	
	To improve generalization when learning from small amounts of data, we devise a novel geometrical formulation of the soft margin SVM loss to neural networks. This loss aims at directly increasing the \emph{input space} margin, in contrast to standard deep network hinge losses that deal with output space margin \cite{tang2013deep,sun2015large}. The authors of \cite{elsayed2018large} also suggest to increase input space margin to improve generalization.
	Figure \ref{fig:teaser} shows 2D examples of training our adversarial robustness and geometric SVM losses for networks.

	In a different application, we use our method for the reconstruction of manifolds such as curves and surfaces in $\Real^3$ from point cloud data. The usage of neural networks for the modeling of surfaces has recently become popular \cite{groueix2018papier, williams2018deep, chen2018learning, ben2018multi,park2019deepsdf,mescheder2018occupancy}. There are two main approaches: parametric and implicit. The parametric approach uses networks as parameterization functions to the manifolds \cite{groueix2018papier,williams2018deep}. The implicit approach represents the manifold as a level set of a neural network \cite{park2019deepsdf, chen2018learning, mescheder2018occupancy}. So far, implicit representations were learned using regression, given a signed distance function or occupancy function computed directly from a ground truth surface. Unfortunately, for raw point clouds in $\Real^3$, computing the signed distance function or an occupancy function is a notoriously difficult task \cite{berger2017survey}. In this paper we show that by using our sample network to control the neural level sets we can reconstruct curves and surfaces directly from point clouds in $\Real^3$.
	
	Lastly, to theoretically justify neural level sets for modeling manifolds or arbitrary decision boundaries, we prove a geometric version of the universality property of MLPs \cite{cybenko1989approximation,hornik1989multilayer}: any piecewise linear hyper-surface in $\Real^d$ (\ie, a $d-1$ manifold built out of a finite number of linear pieces, not necessarily bounded) can be precisely represented as a neural level set of a suitable MLP. 
	
	\newpage
	
	%The contributions of the paper are the following: 
	%\begin{itemize}
	%    \item We devise a novel method for directly controlling level sets of neural networks.
	%    \item We prove a universality theorem, stating that level sets of neural networks are able to represent any polyhedral surface.
	%    \item We validate our approach on three popular tasks: (i) training robust classification networks, (ii)  improving generalization to unseen data and (iii) curve and surface reconstruction from point clouds.
	%\end{itemize}

	%-------------------------------------------------------------------------
	%\vspace{-5pt}
	\section{Sample Network}\label{s:sample_network}\vspace{-5pt}
	Given a neural network $F(x ; \theta):\Real^d\times \Real^{m}\too \Real^l$ its $0\in \Real^l$ level set is defined by
	\begin{equation}
	\gS(\theta):=\set{ x \ \Big \vert \ F(x;\theta)=0}.
	\end{equation} 
	We denote by $D_x F(p;\theta) \in \Real^{l\times d}$ the matrix of partial derivatives of $F$ with respect to $x$. Assuming that $\theta$ is fixed, $F(p;\theta)=0$ and that $D_x F(p;\theta)$ is of full rank $l$ ($l \ll d$), a corollary of the Implicit Function Theorem \cite{krantz2012implicit} implies that $\gS(\theta)$ is a $d-l$ dimensional manifold in the vicinity of $p\in\gS(\theta)$.

	%If $F$'s activations are ReLU, it is locally linear, \ie, $F\vert_P=\ell$, where $P\subset \Real^d$ is a polytope and $\ell(x):\Real^d\too \Real^l$ is a linear function. In this case, if $\ell(p)=0$ for some $p\in P$, and the constant matrix $D_x\ell$ is of full rank $l$, then $\gS(\theta)\cap P$ is the restriction of an affine space of dimension $d-l$ to $P$. \niv{I missed the point of this whole discussion on local linearity. Why are we putting it here? it doesn't seem to relate to the next paragraph at all}
	
	Our goal is to incorporate the neural level set $\gS(\theta)$ in a differentiable loss. We accomplish that by performing the following procedure at each training iteration: (i) Sample $n$ points on the level set: $p_i \in \gS(\theta)$, $i\in [n]$; (ii) Build the sample network $p_i(\theta)$, $i\in [n]$, by adding a fixed linear layer to the network $F(x;\theta)$; and (iii) Incorporate the sample network in a loss function as a proxy for $\gS(\theta)$.
	
	\vspace{-5pt}
	\subsection{Level set sampling} \label{ss:newton}\vspace{-5pt}
	To sample $\gS(\theta)$ at some $\theta=\theta_0$, we start with a set of $n$ points $p_i,\ i \in [n]$ sampled from some probability measure in $\Real^d$. Next, we perform generalized Newton iterations \citep{ben1966newton} to move each point $p$ towards $\gS(\theta_0)$:
	\begin{equation}\label{e:newton}
	p^{\mathrm{next}}=p - D_x F(p; \theta_0)^\dagger F(p;\theta_0), 
	\end{equation} 
	where $D_x F(p, \theta_0)^\dagger$ is the Moore-Penrose pseudo-inverse of $D_x F(p, \theta_0)$.
	The generalized Newton step solves the under-determined ($l \ll d$) linear system $F(p;\theta_0) + D_x F(p;\theta_0)(p^{\mathrm{next}}- p)=0$. To motivate this particular solution choice we show that the generalized Newton step applied to a linear function is an orthogonal projection onto its zero level set (see proof in %the supplementary material):
	Appendix~\ref{app:proofs}):
	\begin{lemma}\label{lem:newton}
		Let $\ell(x)=Ax+b$, $A\in \Real^{l\times d}$, $b\in \Real^l$, $\ell < d$, and $A$ is of full rank $l$. Then \Eqref{e:newton} applied to $F(x)=\ell(x)$ is an orthogonal projection on the zero level set of $\ell$, namely, on $\set{x \ \vert \ \ell(x)=0}$.   
	\end{lemma}
	
	%Neural network with ReLU activation are piecewise linear and therefore Lemma \ref{lem:newton} implies that (in the generic case) if generalized Newton is used sufficiently close to $\gS(\theta_0)$ it would converge in one step to the Euclidean closest point on the level set. 
	
	For $l>1$ the computation of $D_x F(p;\theta_0)^\dagger$ requires inverting an $l \times l$ matrix; in this paper $l\in \set{1,2}$. 
	The directions $D_x F(p_i;\theta_0)$ can be computed efficiently using back-propagation where the points $p_i$, $i\in[n]$ are grouped into batches. We performed $10-20$ iterations of \Eqref{e:newton} for each $p_i$, $i\in[n]$. 
	
	%In practice, we find that $c_i=0$ in most cases.

	%\begin{wrapfigure}[6]{r}{0.19\textwidth}
	%    \centering
	%    \includegraphics[width=0.19\textwidth]{figs/projection.pdf}
	%\end{wrapfigure}

	\paragraph{Scalar networks.}
	For scalar networks, \ie, $l=1$, $D_x F\in \Real^{1\times d}$, a direct computation shows that 
	\begin{equation}\label{e:D_dagger_for_l_equals_1}
	D_x F(p;\theta_0)^\dagger = \frac{D_x F(p;\theta_0)^T}{\norm{D_x F(p;\theta_0)}^2}.  
	\end{equation}
	That is, the point $p$ moves towards the level set $\gS(\theta_0)$ by going in the direction of the steepest descent (or ascent) of $F$.

	It is worth mentioning that the projection-on-level-sets formula in the case of $l=1$ has already been developed in \citep{moosavi2016deepfool} and was used to find adversarial perturbations; our result generalizes to the intersection of several level sets and shows that this procedure is an instantiation of the generalized Newton algorithm.  
	
	The generalized Newton method (similarly to Newton method) is not guaranteed to find a point on the zero level set. We denote by $c_i=F(p_i;\theta_0)$ the level set values of the final point $p_i$; in the following, we use also points that failed to be projected with their level set $c_i$. Furthermore, we found that the generalized Newton method usually does find zeros of neural networks but these can be far from the initial projection point. Other, less efficient but sometimes more robust ways to project on neural level sets could be using gradient descent on $\norm{F(x;\theta)}$ or zero finding in the direction of a successful PGD attack \cite{kurakin2016adversarial,madry2017towards} as done in \cite{ding2018max}. In our robust networks application (Section \ref{sec:adv}) we have used a similar approach with the false position method. %\niv{not sure I understand what false position method mean} 
	
	\paragraph{Relation to Elsayed \etal~\cite{elsayed2018large}.} The authors of \cite{elsayed2018large} replace the margin distance with distance to the linearized network, while our approach is to directly sample the actual network's level set and move it explicitly. Specifically, for $L_2$ norm ($p=2$) Elsayed's method is similar to our method using a single generalized Newton iteration, \Eqref{e:newton}.

	\vspace{-5pt}
	\subsection{Differentiable sample position}\vspace{-5pt}
	Next, we would like to relate each sample $p$, belonging to some level set $F(p;\theta_0)=c$, to the network parameters $\theta$. Namely, $p=p(\theta)$. The functional dependence of a sample $p$ on $\theta$ is defined by $p(\theta_0)=p$ and $F(p(\theta);\theta) = c$, for $\theta$ in some neighborhood of $\theta_0$. The latter condition ensures that $p$ stays on the $c$ level set as the network parameters $\theta$ change. As only first derivatives are used in the optimization of neural networks, it is enough to replace this condition with its first-order version. We get the following two equations:
	\begin{equation}\label{e:necessary_cond}
	p(\theta_0)=p \quad ; \quad  \frac{\partial }{\partial \theta}\Big\vert_{\theta=\theta_0}F(p(\theta);\theta)=0.  
	\end{equation}
	
	%Using the chain rule these are equivalent to 
	%\begin{equation}\label{e:necessary_cond_1st}
	%   p_i(\theta_0)=p_i \quad ; \quad  D_{x} F(p_i,\theta_0) D_\theta p_i(\theta_0) + D_\theta F(p_i,\theta_0)=0.  
	%\end{equation}
	
	Using the chain rule, the second condition in \Eqref{e:necessary_cond} reads:
	\begin{equation}\label{e:dev_chain}
	D_{x} F(p,\theta_0) D_\theta p(\theta_0) + D_\theta F(p,\theta_0)=0.
	\end{equation}
	This is a system of linear equations with $d\times m$ unknowns (the components of $D_\theta p(\theta_0)$) and $l\times m$ equations. When $d > l$, this linear system is under-determined. Similarly to what is used in the generalized Newton method, a minimal norm solution is given by the Moore-Penrose inverse: 
	\begin{equation}\label{e:min_norm_sol}
	D_\theta p(\theta_0) = -D_{x} F(p,\theta_0)^\dagger  D_\theta F(p,\theta_0).
	\end{equation}
	The columns of the matrix $D_\theta p(\theta_0) \in \Real^{d\times m}$ describe the velocity of $p(\theta)$ w.r.t.~each of the parameters in $\theta$. The pseudo-inverse selects the minimal norm solution that can be shown to represent, in this case, a movement in the orthogonal direction to the level set (see %supplementary material 
	Lemma~\ref{app:proofs}\ref{app:lemma_velocity} for a proof). We reiterate that for scalar networks, where $l=1$, $D_x F(p_i;\theta_0)^\dagger$ has a simple closed-form expression, as shown in \Eqref{e:D_dagger_for_l_equals_1}. 
	
	\paragraph{The sample network.} A possible simple solution to \Eqref{e:necessary_cond} would be to use the linear function $p(\theta)=p+D_\theta p(\theta_0)(\theta-\theta_0)$, with $D_\theta p(\theta_0)$ as defined in \Eqref{e:min_norm_sol}. Unfortunately, this would require storing $D_\theta p(\theta_0)$, using at-least $\gO(m)$ space (\ie, the number of network parameters), for every projection point $p$. (We assume the number of output channels $l$ of $F$ is constant, which is the case in this paper.) A much more efficient solution is 
	\begin{equation}\label{e:p}
	p(\theta) = p - D_x F(p;\theta_0)^\dagger \brac{ F(p;\theta) - c },
	\end{equation}
	that requires storing $D_x F(p;\theta_0)^\dagger$, using only $\gO(d)$ space, where $d$ is the input space dimension, for every projection point $p$.
	Furthermore, \Eqref{e:p} allows an efficient implementation with a single network $$G(p,D_x F(p;\theta_0)^\dagger ; \theta):=p(\theta).$$ We call $G$ the \emph{sample network}. Note that a collection of samples $p_i$, $i\in [n]$ can be treated as a batch input to $G$. 
	
	%\Eqref{e:p} is indeed solution to \Eqref{e:necessary_cond} since $p(\theta_0)=p$ and 

	%
	%A key advantage of \Eqref{e:p} is that each sample $p_i(\theta)$, $i\in [n]$, can be constructed by simply adding a single constant linear layer $y\mapsto -A_i^\dagger y + (p_i+ A_i^\dagger c_i)$ where $A_i=D_x F(p_i;\theta_0)$, to the network $F$. We call $P(\theta)=\brac{p_1(\theta),\ldots, p_n(\theta)}:\Real^m\too \Real^{d\times n}$ the \emph{sample network}. 
	%

	\vspace{-5pt}         
	\section{Incorporation of samples in loss functions}\vspace{-5pt} 
	Once we have the sample network $p_i(\theta)$, $i\in [n]$, we can incorporate it in a loss function to control the neural level set $\gS(\theta)$ in a desired way. We give three examples in this paper.

	\vspace{-5pt}
	\subsection{Geometric SVM}\vspace{-5pt}
	Support-vector machine (SVM) is a model which aims to train a linear binary classifier that would generalize well to new data by combining the hinge loss and a large margin term. It can be interpreted as encouraging large distances between training examples and the decision boundary.
	Specifically, the soft SVM loss takes the form \cite{cortes1995support}:
	\begin{equation}\label{e:svm}
	\mathrm{loss}(w,b) =  \frac{1}{N} \sum_{j=1}^{N} \max\set{0, 1-y_j\ell(x_j;w,b)} + \lambda \norm{w}_2^2,  
	\end{equation}
	where $(x_j,y_j)\in \Real^{d}\times \set{-1,1}$, $j\in [N]$ is the binary classification training data, and $\ell(x;w,b)=w^Tx+b$, $w\in \Real^d$, $b\in \Real$ is the linear classifier. We would like to generalize \Eqref{e:svm} to a deep network $F:\Real^d\times \Real^m \too \Real$ towards the goal of increasing the network's \emph{input space margin}, which is defined as the minimal distance of training examples to the decision boundary $\gS(\theta)=\set{x\in \Real^d \ \vert \ F(x;\theta)=0}$. Note that this is in strong contrast to standard deep network hinge loss that works with the \emph{output space margin} \cite{tang2013deep,sun2015large}, namely, measuring differences of output logits when evaluating the network on training examples. For that reason, this type of loss function does not penalize small input space margin, so long as it doesn't damage the output-level classification performance on the training data. Using the input margin over the output margin may also provide robustness to perturbations \cite{elsayed2018large}.
	
	We now describe a new, geometric formulation of \Eqref{e:svm}, and use it to define the soft-SVM loss for general neural networks. In the linear case, the following quantity serves as the margin:
	$$\norm{w}_2^{-1}=\dist(\gS(\theta),\gS_1(\theta))=\dist(\gS(\theta),\gS_{-1}(\theta))$$ 
	
	where $\gS_t(\theta)=\set{x\in \Real^d\ \vert \ F(x;\theta)=t }$, and $\dist(\gS(\theta),\gS_t(\theta))$ is the distance between the level sets, which are two parallel hyper-planes.
	In the general case, however, level sets are arbitrary hyper-surfaces which are not necessarily equidistant (\ie, the distance when traveling from $\gS$ to $\gS_t$ does not have to be constant across $\gS$). Hence, for each data sample $x$, we define the following margin function: $$\Delta(x;\theta)=\min \Big \{ \dist \big(p(\theta),\gS_{1}(\theta)\big),\dist\big(p(\theta),\gS_{-1}(\theta)\big)\Big\},$$ where $p(\theta)$ is the sample network of the projection of $x$ onto $\gS(\theta_0)$. Additionally, note that in the linear case: $\abs{w^Tx+b} = \dist(x,\gS(\theta))/\Delta(x;\theta)$. With these definitions in mind, \Eqref{e:svm} can be given the geometric generalized form:
	\begin{equation}\label{e:svm_geo}
	\mathrm{loss}(w,b) = \frac{1}{N}  \sum_{j=1}^{N} \max\set{0, 1-\mathrm{sign}(y_j F(x_j;\theta))\frac{\dist(x_j,p_j)}{\Delta(x_j;\theta)}} + \frac{\lambda}{N}  \sum_{j=1}^{N} \Delta(x_j;\theta)^\alpha ,
	\end{equation}
	where $F(x;\theta)$ is a general classifier (such as a neural network, in our applications). Note that in the case where $F(x;\theta)$ is affine, $\alpha=-2$ and $\dist=L_2$, \Eqref{e:svm_geo} reduces back to the regular SVM loss, \Eqref{e:svm}. Figure \ref{fig:teaser}d depicts the result of optimizing this loss in a 2D case, \ie, $x_j\in\Real^2$; the light blue and red curves represent $\gS_{-1}$ and $\gS_1$.

	\subsection{Robustness to adversarial perturbations}
	\label{sec:adv}
	
	The goal of robust training is to prevent a change in a model's classification result when small perturbations are applied to the input. Following \cite{madry2017towards} the attack model is specified by some set $S \subset \Real^d$ of allowed perturbations; in this paper we focus on the popular choice of $L_\infty$ perturbations, that is $S$ is taken to be the $\varepsilon$-radius $L_{\infty}$ ball, $\set{x\ \vert \ \norm{x}_\infty \leq \varepsilon}$. Let $(x_j,y_j)\in \Real^d \times \gL$ denote training examples and labels, and let $\gS^j(\theta)=\set{x\in\Real^d \ \vert \ F^j(x;\theta)=0}$, where $F^{j}(x;\theta) = f_{j}-\max_{i\ne j}f_i$, the decision boundary of label $j$. We define the loss 
	\begin{equation}\label{e:adversarial}
	\mathrm{loss}(\theta)=\frac{1}{N}\sum_{j=1}^N \lambda_j \max\Big\{ 0, \varepsilon_j - \mathrm{sign}(F^{y_j}(x_j;\theta))\dist(x_j,\gS^{y_j}(\theta)) \Big \},
	\end{equation} 
	where $\dist(x,\gS^j)$ is some notion of a distance between $x$ and $\gS^j$, \eg, $\min_{y\in \gS^j}\norm{x-y}_p$ or $\int_{\gS^j(\theta)} \norm{x-y}_p d\mu(y)$, $d\mu$ is some probability measure on $\gS^j(\theta)$.
	The parameter $\lambda_j$ controls the weighting between correct (\ie, $F^{y_j}(x_j;\theta)>0$) and incorrect (\ie, $F^{y_j}(x_j;\theta)<0$) classified samples. We fix $\lambda_j = 1$ for incorrectly classified samples and set $\lambda_j$ to be the same for correctly classified samples;
	The parameter $\varepsilon_j$ controls the desired target distances; 
	% we used the same $\lambda_j$ and $\varepsilon_j$ for all correctly classified (\ie, $F^{y_j}(x_j;\theta)>0$) and wrongly classified (\ie, $F^{y_j}(x_j;\theta)<0$) examples.
	%
	Similarly to \cite{elsayed2018large}, the idea of this loss is: (i) if $x_j$ is incorrectly classified, pull the decision boundary $\gS^{y_j}$ toward $x_j$; (ii) if $x_j$ is classified correctly, push the decision boundary $\gS^{y_j}$ to be within a distance of at-least $\varepsilon_j$ from $x_j$.

	\begin{wrapfigure}[4]{r}{0.1\textwidth}
		\vspace{-15pt}
		\begin{center}
			\includegraphics[width=0.1\textwidth]{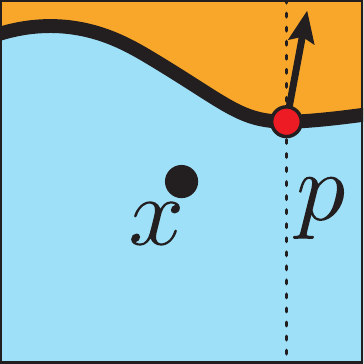}
		\end{center}
	\end{wrapfigure}
	In our implementation we have used $\dist(x,\gS^j)=\rho(x,p(\theta))$, where $p=p(\theta_0)\in \gS^j$ is a sample of this level set; $\rho(x,p)=\abs{x^{i_*}-p^{i_*}}$, and $x^{i_*},p^{i_*}$ denote the $i_*$-th coordinates of $x,p$ (resp.), $i_* = \argmax_{i\in [d]} \abs{D_x F(p;\theta_0)}$. This loss encourages $p(\theta)$ to move in the direction of the axis (\ie, $i_*$) that corresponds to the largest component in the gradient $D_x F(p;\theta_0)$. Intuitively, as $\rho(x,p)\leq \norm{x-p}_\infty$, the loss pushes the level set $\gS^j$ to leave the $\varepsilon_j$-radius $L_\infty$-ball in the direction of the axis that corresponds to the maximal speed of $p(\theta)$. The inset depicts an example where this distance measure is more effective than the standard $L_\infty$ distance: $D_x F(p;\theta)$ is shown as black arrow and the selected axis $i_*$ as dashed line.

	In the case where the distance function $\min_{y\in \gS(\theta)}\norm{x-y}_p$ is used, the computation of its gradient using \Eqref{e:min_norm_sol} coincides with the gradient derivation of \cite{ding2018max} up to a sign difference. Still, our derivation allows working with general level set points (\ie, not just the closest) on the decision boundary $\gS^j$, and our sample network offers an efficient implementation of these samples in a loss function. Furthermore, we use the same loss for both correctly and incorrectly classified examples. 
	
	%we employ a different loss function that moves the decision boundary in the correct direction for incorrectly classified examples as well.

	\subsection{Manifold reconstruction}
	\paragraph{Surface reconstruction.} Given a point cloud $\gX = \set{x_j}_{j=1}^N \subset \Real^d$ that samples, possibly with noise, some surface $\gM\subset \Real^3$, our goal is to find parameters $\theta$ of a network $F:\Real^3\times \Real^m\too \Real$, so that the neural level set $\gS(\theta)$ approximates $\gM$. Even more desirable is to have $F$ approximate the signed distance function to the unknown surface sampled by $\gX$. To that end, we would like the neural level set $\gS_t(\theta)$, $t \in \gT$ to be of distance $|t|$ to $\gX$, where $\gT \subset \Real$ is some collection of desired level set values. Let $\dist(x,\gX)=\min_{j\in[N]} \norm{x-x_j}_2$ be the distance between $x$ and $\gX$. 
	We consider the reconstruction loss
	%\begin{equation}
	%    \mathrm{loss}(\theta) = \brac{\int_{\gS(\theta)} \dist(x,\gX)^p   d\mu(x)}^{\frac{1}{p}} + %\frac{\lambda}{N} \sum_{j=1}^N \abs{F(x_j;\theta)}, 
	%\end{equation}
	\begin{equation}\label{e:reconstruct}
	\mathrm{loss}(\theta) = \sum_{t\in\gT}\brac{\int_{\gS_t(\theta)} \Big |\dist(x,\gX)-\abs{t}\Big |^p dv(x)}^{\frac{1}{p}} + \frac{\lambda}{N} \sum_{j=1}^N \abs{F(x_j;\theta)}, 
	\end{equation}
	where $dv(x)$ is the normalized volume element on $\gS_t(\theta)$ and $\lambda>0$ is a parameter. The first part of the loss encourages the $t$ level set of $F$ to be of distance $|t|$ to $\gX$; note that for $t=0$ this reconstruction error was used in level set surface reconstruction methods \cite{zhao2001fast}. The second part of the loss penalizes samples $\gX$ outside the zero level set $\gS(\theta)$. 
	
	\paragraph{Curve reconstruction.} In case of approximating a manifold $\gM\subset \Real^d$ with co-dimension greater than $1$, \eg, a curve in $\Real^3$, one cannot expect $F$ to approximate the signed distance function as no such function exists. Instead, we model the manifold via the level set of a vector-valued network $F:\Real^d\times \Real^m \too \Real^l$ whose zero level set is an intersection of $l$ hyper-surfaces. As explained in Section \ref{s:sample_network}, this generically defines a $d-l$ manifold. In that case we used the loss in \Eqref{e:reconstruct} with $\gT=\set{0}$, namely, only encouraging the zero level set to be as close as possible to the samples $\gX$. 
	%-------------------------------------------------------------------------
	
	\vspace{-5pt}
	\section{Universality}\vspace{-5pt}
	To theoretically support the usage of neural level sets for modeling manifolds or controlling decision boundaries we provide a geometric universality result for multilayer perceptrons (MLP) with ReLU activations. That is, the level sets of MLPs can represent any watertight piecewise linear hyper-surface (\ie, manifolds of co-dimension $1$ in $\Real^d$ that are boundaries of $d$-dimensional polytopes). More specifically, we prove:
	\begin{theorem}\label{e:levelset_universality}
		Any watertight, not necessarily bounded, piecewise linear hypersurface $\gM\subset \Real^d$ can be exactly represented as the neural level set $\gS$ of a multilayer perceptron with ReLU activations, $F:\Real^{d}\too\Real$. \vspace{-5pt}
	\end{theorem}
	The proof of this theorem is given %in the supplementary material 
	in Appendix~\ref{app:proofs}. Note that this theorem is a geometrical version of Theorem 2.1 in \cite{arora2016understanding}, asserting that MLPs with ReLU activations can represent any piecewise linear continuous function.

	%-------------------------------------------------------------------------
	\vspace{-5pt}
	\section{Experiments}\vspace{-5pt}
	\begin{wrapfigure}[7]{r}{0.6\textwidth}
		\vspace{-55pt}
		\begin{tabular}{@{}c@{}c@{}c@{}}
			\includegraphics[width=0.2\columnwidth]{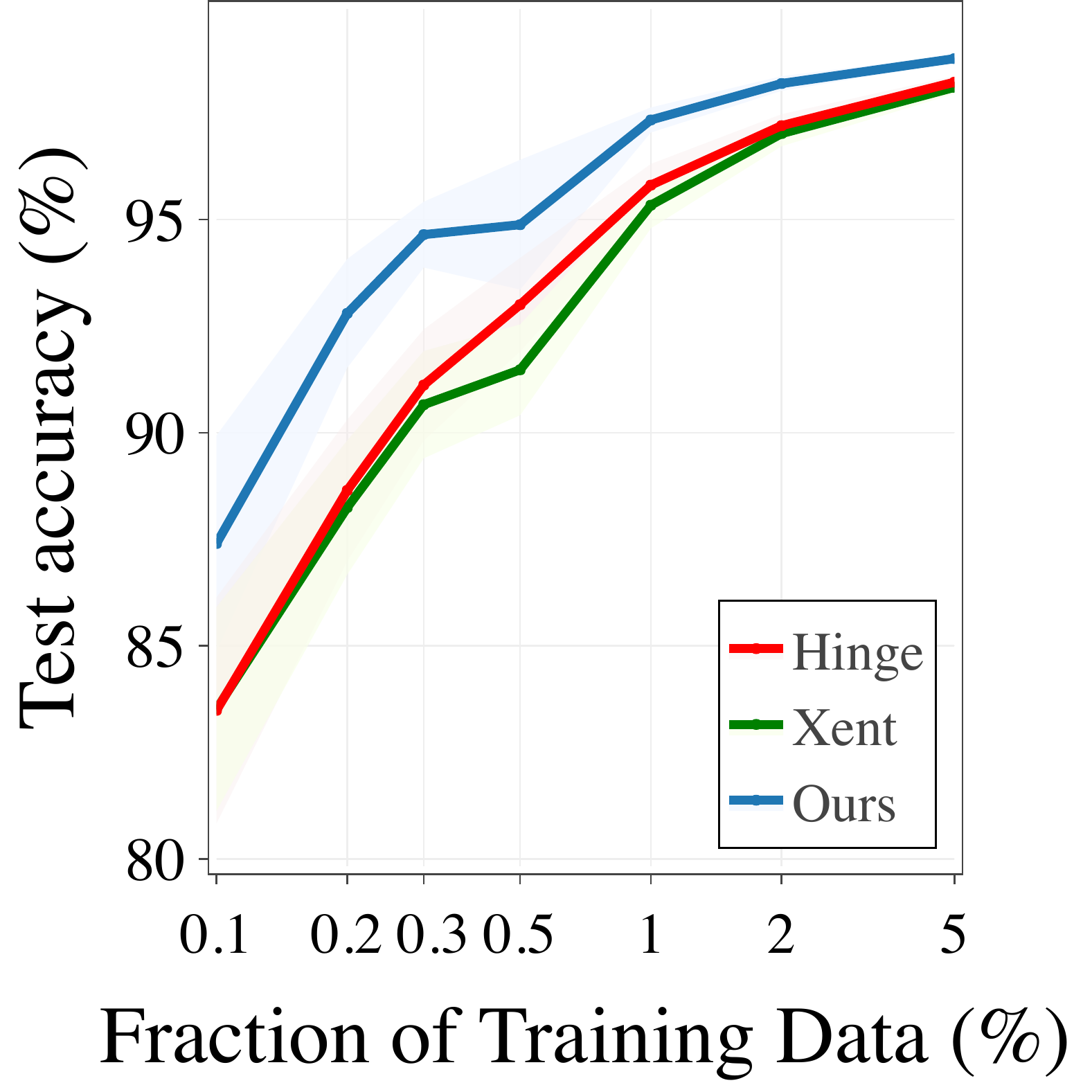} &
			\includegraphics[width=0.2\columnwidth]{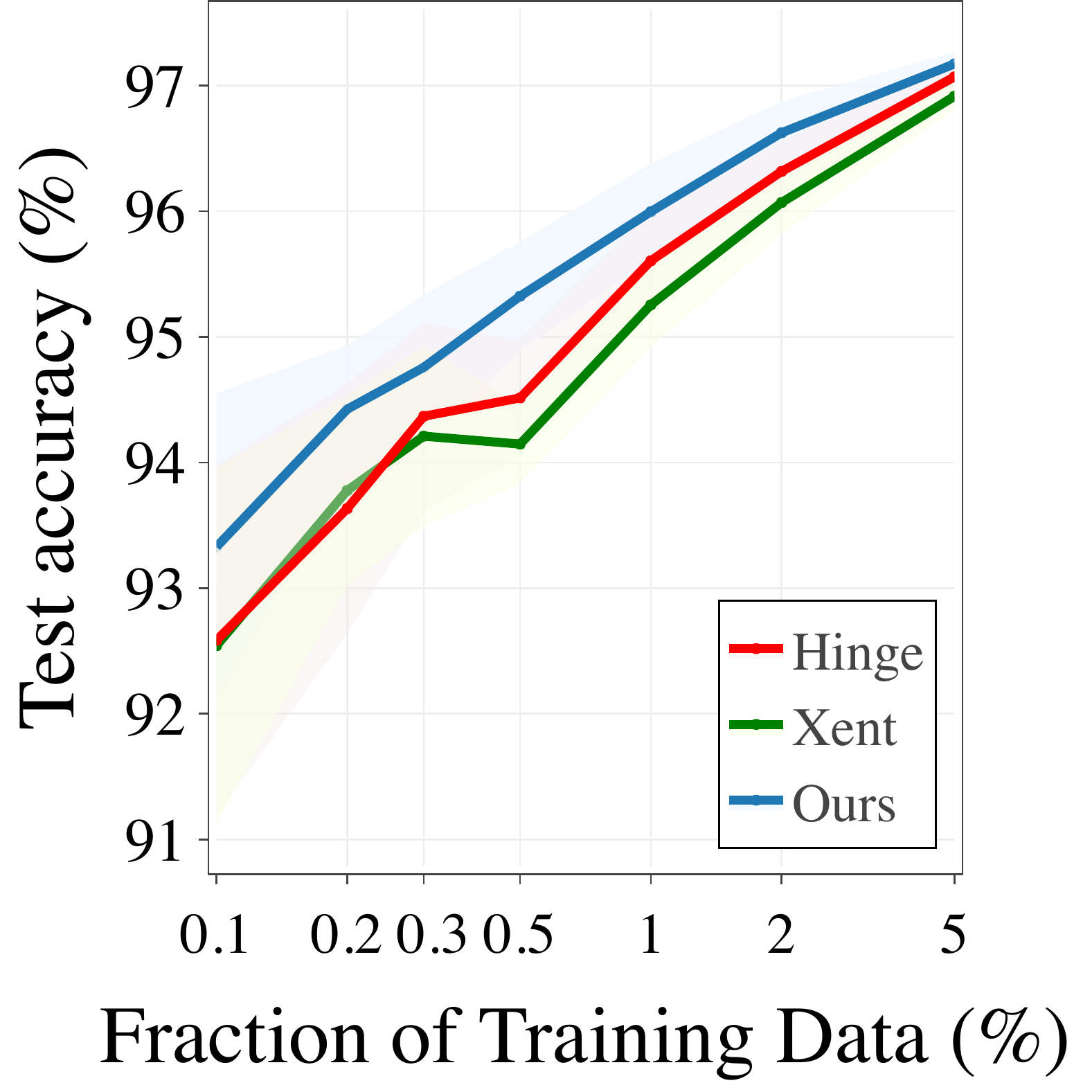} &
			\includegraphics[width=0.2\columnwidth]{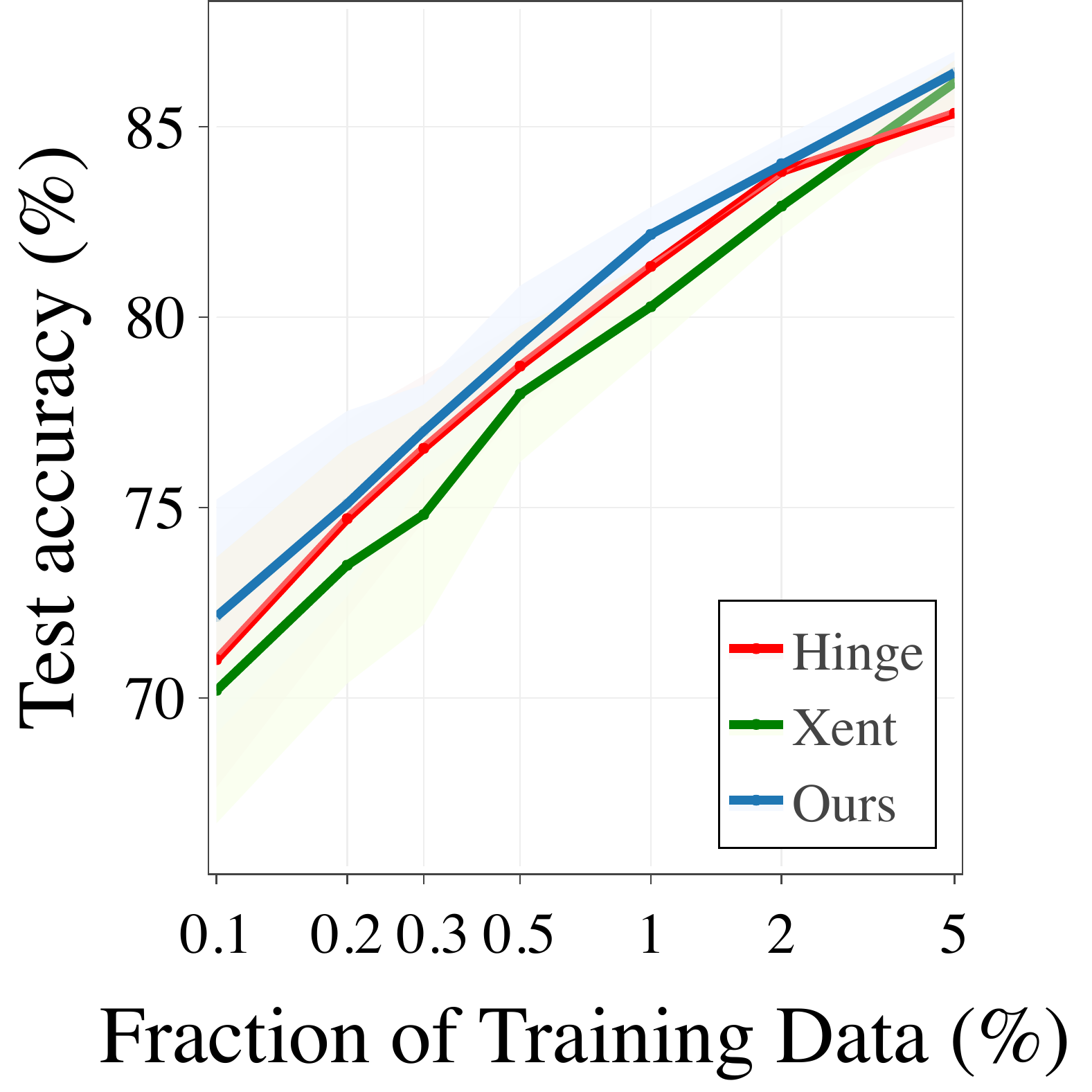} \\
			(a) & (b) & (c) 
		\end{tabular}\vspace{-5pt}
		\caption{Generalization from small fractions of the data.}\label{fig:generalization}
	\end{wrapfigure}
	\subsection{Classification generalization} 
	\label{sec:exp_clasgen}
	
	In this experiment, we show that when training on small amounts of data, our geometric SVM loss (see \Eqref{e:svm_geo}) generalizes better than the cross entropy loss and the hinge loss. 
	Experiments were done on three datasets: MNIST \cite{lecun1998mnist}, Fashion-MNIST \cite{xiao2017fashion} and CIFAR10 \cite{krizhevsky2009learning}. For all datasets we arbitrarily merged the labels into two classes, resulting in a binary classification problem. We randomly sampled a fraction of the original training examples and evaluated on the original test set.

	Due to the variability in the results, we rerun the experiment 100 times for MNIST and 20 times for Fashion-MNIST and CIFAR10. We report the mean accuracy along with the standard deviation. Figure \ref{fig:generalization} shows the test accuracy of our loss compared to the cross-entropy loss and hinge loss over different training set sizes for MNIST (a), Fashion-MNIST (b) and CIFAR10 (c). Our loss function outperforms the standard methods.

	For the implementation of \Eqref{e:svm_geo} we used $\alpha=-1$, $\dist = L_\infty$, and approximated $\dist(x,\gS_t)\approx \norm{x-p^t}_\infty$, where $p^t$ denotes the projection of $p$ on the level set $\gS_t$, $t\in\set{-1,0,1}$ (see Section \ref{ss:newton}). The approximation of the term $\Delta(x;\theta)$, where $x=x_j$ is a train example, is therefore $\min\set{ \|p^0-p^{-1}\|_\infty, \|p^0-p^1\|_\infty }$. See %supplementary material 
	Appendix~\ref{app:imp_gen} for further implementation details.

	%as it requires computing $\dist\big(p_x(\theta),\gS_{1}(\theta)\big)$, $\dist\big(p_x(\theta),\gS_{-1}(\theta)\big)$, where $p_x$ is the projection of training example $x$ to $\gS_{0}$. To do so, we additionally project the training example $x$ on the level sets $\gS_{-1}$ and $\gS_{1}$, denoted by $p_x^{-}$ and $p_x^{+}$ respectively. Following that, we use the approximation $\Delta(x;\theta) \approx \min \set{\left{ \dist\left(p_x,p_x^{-}\right),\dist\left(p_x,p_x^{+}\right)\right}}$.

	\subsection{Robustness to adversarial examples}
	\label{sec:exp_adv}
	
	In this experiment we used our method with the loss in \Eqref{e:adversarial} to train robust models on MNIST \cite{lecun1998mnist} and CIFAR10 \cite{krizhevsky2009learning} datasets. For MNIST we used ConvNet-4a (312K params) used in \cite{zhang2019theoretically} and for CIFAR10 we used two architectures: ConvNet-4b (2.5M params) from \cite{wong2018scaling} and ResNet-18 (11.2M params) from \cite{zhang2019theoretically}. 
	%
	%We used the same networks as in \cite{wong2018scaling} (MNIST small and CIFAR10 large, see Table 1 in their paper). 
	%
	We report results using the loss in \Eqref{e:adversarial} with the choice of $\varepsilon_j$ fixed as $\varepsilon_\text{train}$ in Table \ref{tab:robustness_results},  $\lambda_j$ to be $1,10$ for MNIST and CIFAR10 (resp.), and $\dist=\rho$ as explained in Section \ref{sec:adv}. 
	We evaluated our trained networks on $L_\infty$ bounded attacks with $\varepsilon_\text{attack}$ radius using Projected Gradient Descent (PGD) \cite{kurakin2016adversarial,madry2017towards} and compared to networks with the same architectures trained using the methods of Madry \etal~\cite{madry2017towards} and TRADES \cite{zhang2019theoretically}.
	We found our models to be robust to PGD attacks based on the Xent loss; during the revision of this paper we discovered weakness of our trained models to PGD attack based on the margin loss, \ie, $\min \set{F^j}$, where $F^j=f_j-\max_{i\ne j} f_i$; we attribute this fact to the large gradients created at the level set. Consequently, we added margin loss attacks to our evaluation.  The results are summarized in Table \ref{tab:robustness_results}. Note that although superior in robustness to Xent attacks we are comparable to baseline methods for margin attacks. Furthermore, we believe the relatively low robust accuracy of our model when using the ResNet-18 architecture is due to the fact that we didn't specifically adapt our method to Batch-Norm layers. 
	
	%In the supplementary material 
	In Appendix~\ref{app:adv_rob} we provide tables summarizing robustness of our trained models (MNIST ConvNet-4a and CIFAR10 ConvNet-4b) to black-box attacks; we log black-box attacks of our and baseline methods \cite{madry2017towards,zhang2019theoretically} in an all-versus-all fashion. In general, we found that all black-box attacks are less effective than the relevant white-box attacks, our method performs better when using standard model black-box attacks, and that all three methods compared are in general similar in their black-box robustness.

	% \yl{remove?} In Figure~\ref{fig:adv_epss} we show the robustness of our trained models to norm-bounded attacks with growing bounds ($\varepsilon_\text{attack} \in \brac{0,1}$). Note that our trained models are significantly more robust than competing methods in cases $\varepsilon_\text{attack}>\varepsilon$. Implementation details and more results can be found in the supplementary materials.

	%However, they are much faster to train and evaluate.

	\begin{table}[t]
		\centering
		\setlength\tabcolsep{2pt} % default value: 6pt
		\begin{tabular}{c}
			\begin{adjustbox}{max width=\textwidth}
				\aboverulesep=0ex
				\belowrulesep=0ex
				\renewcommand{\arraystretch}{1.1}
				\begin{tabular}[t]{l||l|l|l|c|c|c|c}
					\toprule
					Method                                  &
					Dataset & Arch.  & Attack  & $\varepsilon_\text{train}$ & Test Acc. & Rob. Acc. Xent & Rob. Acc. Margin \\
					\midrule
					Standard                                & 
					MNIST & ConvNet-4a &     PGD$^{40}$($\varepsilon_\text{attack}=0.3$) & - &
					99.34\% &  13.59\% & 0.00\%   \\
					Madry et al. \cite{madry2017towards}    & 
					MNIST & ConvNet-4a &     PGD$^{40}$($\varepsilon_\text{attack}=0.3$) & $0.3$ &
					99.35\% &  96.04\% & 96.11\%   \\
					Madry et al. \cite{madry2017towards}    & 
					MNIST & ConvNet-4a &     PGD$^{40}$($\varepsilon_\text{attack}=0.3$) & $0.4$ &
					99.16\% &  96.54\% & 96.53\%   \\
					TRADES \cite{zhang2019theoretically}    & 
					MNIST & ConvNet-4a &     PGD$^{40}$($\varepsilon_\text{attack}=0.3$) & $0.3$ &
					98.97\% &  96.75\% & 96.74\%   \\
					TRADES \cite{zhang2019theoretically}    & 
					MNIST & ConvNet-4a &     PGD$^{40}$($\varepsilon_\text{attack}=0.3$) & $0.4$ &
					98.62\% &  96.78\% & 96.76\%   \\
					Ours                                    &
					MNIST & ConvNet-4a &     PGD$^{40}$($\varepsilon_\text{attack}=0.3$) & $0.4$ &
					99.35\% &  99.23\% & 97.35\%   \\
					\midrule
					Standard                               & 
					CIFAR10 & ConvNet-4b &  PGD$^{20}$ ($\varepsilon_\text{attack}=0.031$)& - &
					83.67\% & 0.00\% & 0.00\%        \\
					Madry et al. \cite{madry2017towards} & 
					CIFAR10 & ConvNet-4b &  PGD$^{20}$ ($\varepsilon_\text{attack}=0.031$)& $0.031$ &
					71.86\% & 39.84\% & 38.18\%        \\
					Madry et al. \cite{madry2017towards} & 
					CIFAR10 & ConvNet-4b & PGD$^{20}$ ($\varepsilon_\text{attack}=0.031$)&  $0.045$ &
					63.66\% & 41.53\% & 39.13\%        \\
					TRADES \cite{zhang2019theoretically}   & 
					CIFAR10 & ConvNet-4b & PGD$^{20}$ ($\varepsilon_\text{attack}=0.031$)&  $0.031$ &
					71.24\% & 41.89\% & 38.4\%        \\
					TRADES \cite{zhang2019theoretically} & 
					CIFAR10 & ConvNet-4b &  PGD$^{20}$ ($\varepsilon_\text{attack}=0.031$)& $0.045$ &
					68.24\% & 42.04\% & 38.18\%        \\
					Ours                               & 
					CIFAR10 & ConvNet-4b & PGD$^{20}$ ($\varepsilon_\text{attack}=0.031$)&  $0.045$ &
					71.96\% & 38.45\% & 38.54\%      \\
					\midrule
					Standard                               & 
					CIFAR10 & ResNet-18 &  PGD$^{20}$ ($\varepsilon_\text{attack}=0.031$)& - &
					93.18\% & 0.00\% & 0.00\%    \\
					Madry et al. \cite{madry2017towards}      & 
					CIFAR10 & ResNet-18 & PGD$^{20}$ ($\varepsilon_\text{attack}=0.031$)&  $0.031$ &
					81.0\% & 47.29\% & 46.58\%    \\
					Madry et al. \cite{madry2017towards}      & 
					CIFAR10 & ResNet-18 &  PGD$^{20}$ ($\varepsilon_\text{attack}=0.031$)& $0.045$ &
					74.97\% & 49.84\% & 48.02\%    \\
					TRADES \cite{zhang2019theoretically}      & 
					CIFAR10 & ResNet-18 &  PGD$^{20}$ ($\varepsilon_\text{attack}=0.031$)& $0.031$ &
					83.04\% & 53.31\% & 51.36\%    \\
					TRADES \cite{zhang2019theoretically}      & 
					CIFAR10 & ResNet-18 &  PGD$^{20}$ ($\varepsilon_\text{attack}=0.031$)& $0.045$ &
					79.52\% & 53.49\% & 51.22\%    \\
					Ours                                 & 
					CIFAR10 & ResNet-18 & PGD$^{20}$ ($\varepsilon_\text{attack}=0.031$)&  $0.045$ &
					81.3\% & 79.74\% & 43.17\%    \\
					%STRESS TESTS?
					\bottomrule
				\end{tabular} 
			\end{adjustbox}
			
		\end{tabular}
		\vspace{3pt}
		\caption{Results of different $L_\infty$-bounded attacks on models trained using our method (described in Section~\ref{sec:adv}) compared to other methods. } %\vspace{-15pt}
		\label{tab:robustness_results}
	\end{table}

	% \begin{figure}[t]
	%     \setlength\tabcolsep{1pt} % default value: 6pt
	%     \begin{tabular}{cccc}
	%      \includegraphics[width=0.25\columnwidth]{figs/paper/adverserial/mnist_1_.pdf} & 
	%      \includegraphics[width=0.25\columnwidth]{figs/paper/adverserial/mnist_50_.pdf} & 
	%      \includegraphics[width=0.25\columnwidth]{figs/paper/adverserial/cifar_1_.pdf} & 
	%      \includegraphics[width=0.25\columnwidth]{figs/paper/adverserial/cifar_20_.pdf} \\
	%      MNIST (1 rest.) &  MNIST (50 rest.) &  CIFAR10 (1 rest.) &  CIFAR10 (20 rest.) \\
	%     %\multicolumn{2}{c}{(a) MNIST} & \multicolumn{2}{c}{(b) CIFAR10}
	% \end{tabular}\vspace{-5pt}
	%     \caption{\yl{remove?} Robustness to increasing size $\varepsilon_\text{attack}$ ($x$-axis) PGD-attacks of models trained using: our method, standard, Madry et al. \cite{madry2017towards} and TRADES \cite{zhang2019theoretically}. In CIFAR10 the $x$-axis is scaled by $255$.}\vspace{-10pt}
	%     \label{fig:adv_epss}
	% \end{figure}

	\begin{figure}
		\centering
		\includegraphics[width=\textwidth]{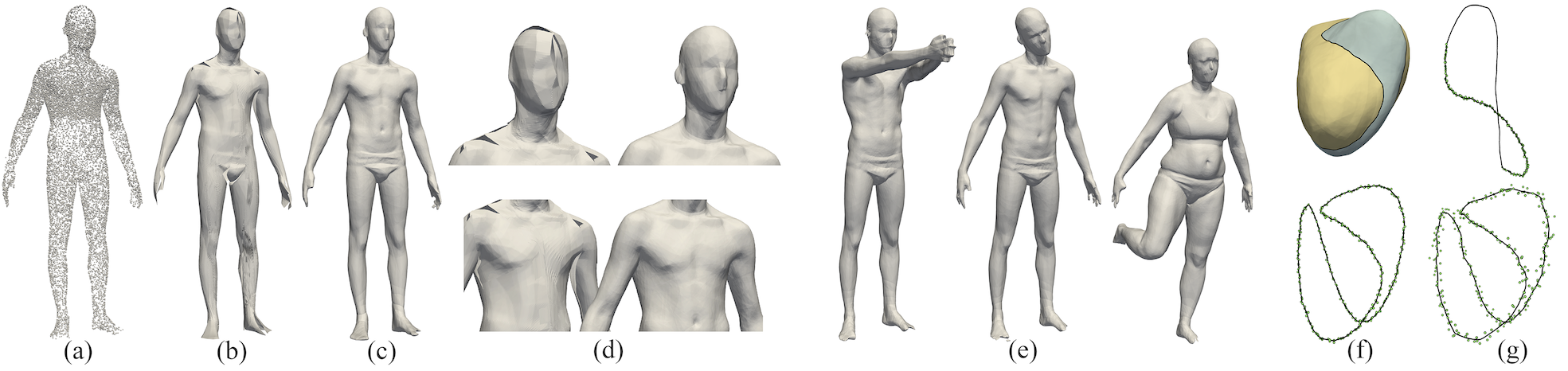}\vspace{-5pt}
		\caption{Point cloud reconstruction. Surface: (a) input point cloud; (b) AtlasNet \cite{groueix2018papier} reconstruction; (c) our result; (d) blow-ups; (e) more examples of our reconstruction. Curve: (f) bottom image shows a curve reconstructed (black line) from a point cloud (green points) as an intersection of two scalar level sets (top image); (g) bottom shows curve reconstruction from a point cloud with large noise, where top image demonstrates the graceful completion of an open curve point cloud data. }\vspace{-12pt}
		\label{fig:recon}
	\end{figure}

	\subsection{Surface and curve reconstruction}
	\label{ss:curves_and_surfaces}
	
	In this experiment we used our method to reconstruct curves and surfaces in $\Real^3$ using only incomplete point cloud data $\gX\subset\Real^3$, which is an important task in 3D shape acquisition and processing. Each point cloud is processed independently using the loss  function described in \Eqref{e:reconstruct}, which encourages the zero level set of the network to pass through the point cloud. For surfaces, it also moves other level sets to be of the correct distance to the point cloud. 
	
	\begin{wraptable}[7]{r}{0.51\columnwidth}\vspace{-12pt}
		\centering
		\resizebox{0.5\textwidth}{!}{    
			\begin{tabular}{lrr} 
				& Chamfer L1                & Chamfer L2 \\
				\midrule
				AtlasNet-1 sphere    & ${23.56 \pm 2.91}$      & ${17.69 \pm 2.45}$   \\
				AtlasNet-1 patch    & ${18.67 \pm 3.45}$      & ${13.38 \pm 2.66}$   \\
				AtlasNet-25 patches   & ${11.54 \pm 0.53}$      & ${7.89 \pm 0.42}$   \\
				Ours        & ${\mathbf{10.71}\pm 0.63 }$    & ${\mathbf{7.32} \pm 0.46}$ \\
			\end{tabular} 
		}\vspace{-0pt}
		\caption{Surface reconstruction results.}
		\label{tab:surface_recon}
	\end{wraptable}
	For surface reconstruction, we trained on 10 human raw scans from the FAUST dataset \cite{Bogo:CVPR:2014}, where each scan consists of $\sim 170$K points in $\Real^3$. The scans include partial connectivity information which we do not use. After convergence, we reconstruct the mesh using the marching cubes algorithm \cite{lorensen1987marching} sampled at a resolution of $\left[100\right]^3$. 
	Table \ref{tab:surface_recon} compares our method with the recent method of \cite{groueix2018papier} which also works directly with point clouds. Evaluation is done using the Chamfer distance \cite{fan2017point} computed between $30$K uniformly sampled points from our and \cite{groueix2018papier} reconstructed surfaces and the ground truth registrations provided by the dataset, with both $L_1, L_2$ norms. Numbers in the table are multiplied by $10^{3}$. We can see that our method outperforms its competitor; Figure \ref{fig:recon}b-\ref{fig:recon}e show examples of surfaces reconstructed from a point cloud (a batch of $10$K points is shown in \ref{fig:recon}a) using our method (in \ref{fig:recon}c, \ref{fig:recon}d-right, \ref{fig:recon}e), and the method of \cite{groueix2018papier} (in \ref{fig:recon}b, \ref{fig:recon}d-left). Importantly, we note that there are recent methods for implicit surface representation using deep neural networks \cite{chen2018learning,park2019deepsdf,mescheder2018occupancy}. These methods use signed distance information and/or the occupancy function of the ground truth surfaces and perform regression on these values. Our formulation, in contrast, allows working directly on the more common, raw input of point clouds. 
	
	For curve reconstruction, we took a noisy sample of parametric curves in $\Real^3$ and used similar network to the surface case, except its output layer consists of two values. We trained the network with the loss \Eqref{e:reconstruct}, where $\gT=\set{0}$, using similar settings to the surface case. Figure \ref{fig:recon}f shows an example of the input point cloud (in green) and the reconstructed curve (in black) (see bottom image), as well as the two hyper-surfaces of the trained network, the intersection of which defines the final reconstructed curve (see top image); \ref{fig:recon}g shows two more examples: reconstruction of a curve from higher noise samples (see bottom image), and reconstruction of a curve from partial curve data (see top image); note how the network gracefully completes the curve.  
	
	\vspace{-8pt}
	\section{Conclusions}\vspace{-8pt}
	
	We have introduced a simple and scalable method to incorporate level sets of neural networks into a general family of loss functions. Testing this method on a wide range of learning tasks we found the method particularly easy to use and applicable in different settings. Current limitations and interesting venues for future work include: applying our method with the batch normalization layer (requires generalization from points to batches); investigating control of intermediate layers' level sets; developing sampling conditions to ensure coverage of the neural level sets; and employing additional geometrical regularization to the neural level sets (\eg, penalize curvature).

	
	% acknowledgments go at the end of the paper. Do not include
	% acknowledgments in the anonymized submission, only in the final paper.

	\subsection*{Acknowledgments}
	This research was supported in part by the European Research Council (ERC Consolidator Grant, "LiftMatch" 771136) and the Israel Science Foundation (Grant No. 1830/17).

	\nocite{paszke2017automatic}
	
	{\small
		\bibliographystyle{ieee}
		\bibliography{levelset}
	}

	\appendix

	% \renewcommand{\thesection}{A\arabic{chapter}}
	
	% % add "A" prefix (for appendix) to tables and figures
	% \renewcommand{\thetable}{A\arabic{table}}
	% \setcounter{table}{0}
	% \renewcommand{\thefigure}{A\arabic{figure}}
	% \setcounter{figure}{0}

	\section{Additional Experiments}
	
	% add "A" prefix (for appendix) to tables and figures
	\renewcommand{\thetable}{A\arabic{table}}
	\setcounter{table}{0}
	\renewcommand{\thefigure}{A\arabic{figure}}
	\setcounter{figure}{0}
	
	\begin{figure}[h!]%{r}{0.3\textwidth}
		\centering
		\vspace{-10pt}
		\hspace{-10pt}
		\begin{minipage}{0.4\textwidth}
			\centering
			\begin{tabular}[t]{ccc}
				Initialization Method  & Chamfer & Hausdorf  \\ \hline
				Uniform [-0.35,0.35] & 0.011 & 0.141 \\
				Normal $\sigma=0.01$ & 0.006 & 0.017\\
				Normal $\sigma=0.05$ & 0.01 & 0.132
			\end{tabular}
		\end{minipage}
		\begin{minipage}{0.6\textwidth}
			\centering
			\begin{tabular}{cc}
				\includegraphics[width=0.25\columnwidth]{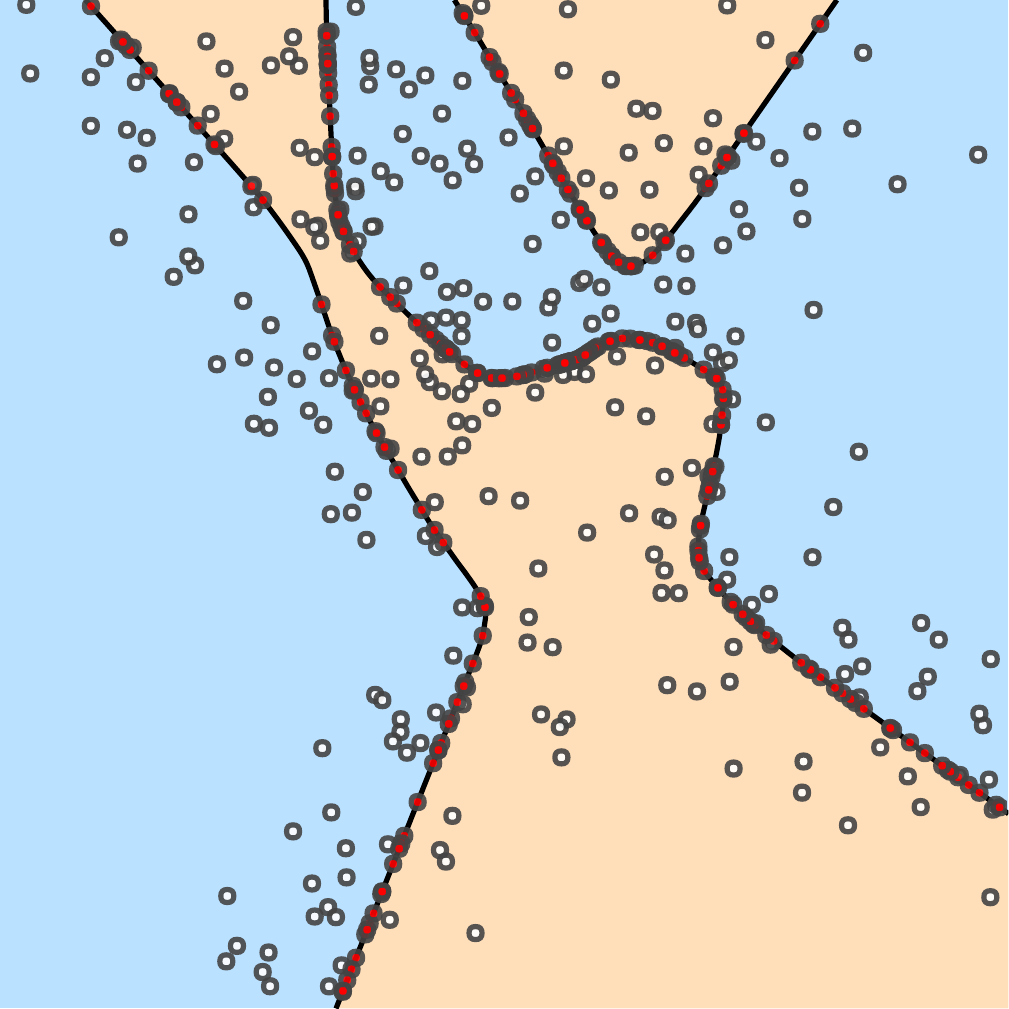} &
				\includegraphics[width=0.25\columnwidth]{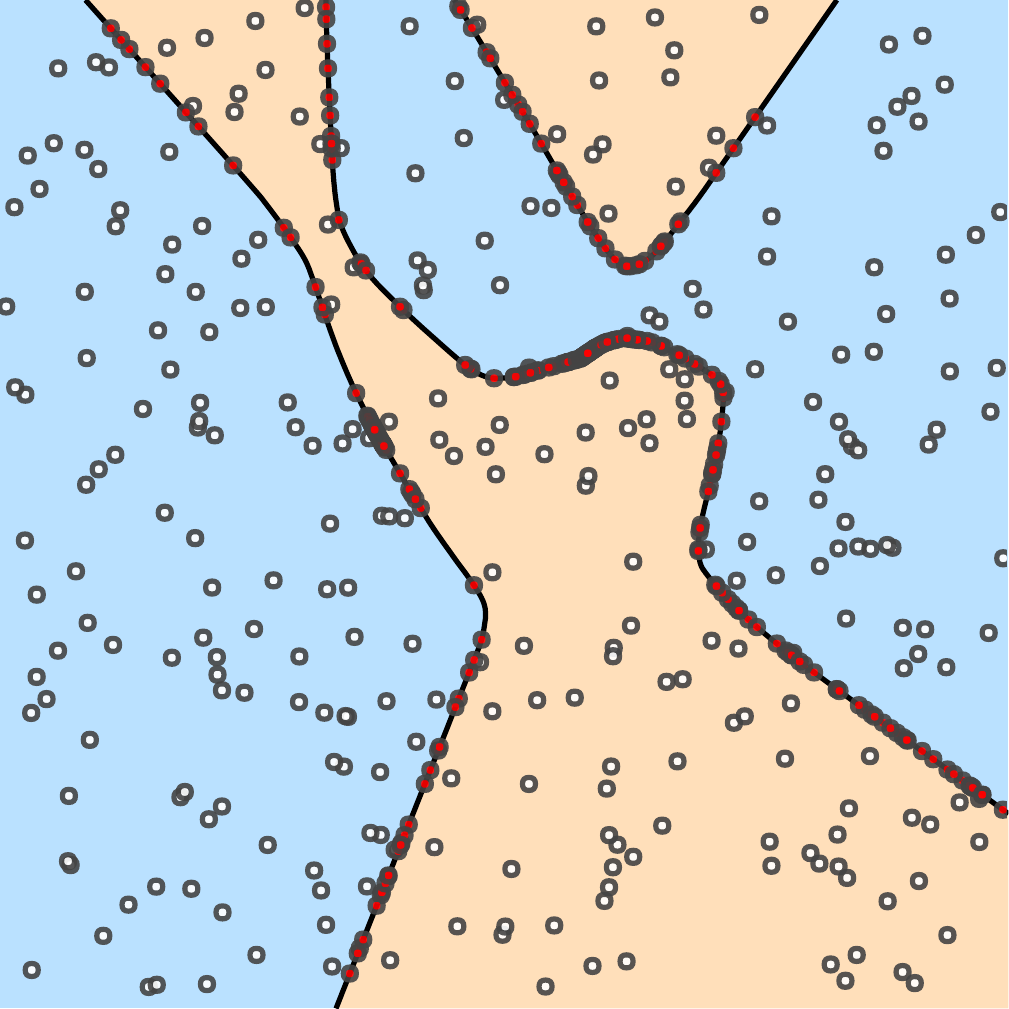} \cr
				{\tiny (a) Normal $\sigma = 0.05$} & {\tiny  (b) Uniform }
			\end{tabular}\vspace{-5pt}
		\end{minipage}
		\caption{Distribution of samples on $d=2$ neural level set.}\label{fig:distribution}
	\end{figure}
	
	\subsection{Distribution of points on the level set.} 
	Achieving well distributed samples of neural level set is a challenge, especially for high dimensions. In the inset we quantify the quality of distribution in low dimension, $d=2$, (where ground truth dense sampling of the level set is tractable). The table in Figure \ref{fig:distribution} logs the Chamfer and Hausdorff distances of the resulting sampling distribution and the level set of a neural network trained with Xent loss in 2 dimensions ((a) and (b)) where projected points (red) are initialized using a uniformly distributed points (gray, (b)) or normally perturbed level set samples (gray, (a)).

	\begin{figure}[h!]
		\centering
		\includegraphics[width=0.4\textwidth]{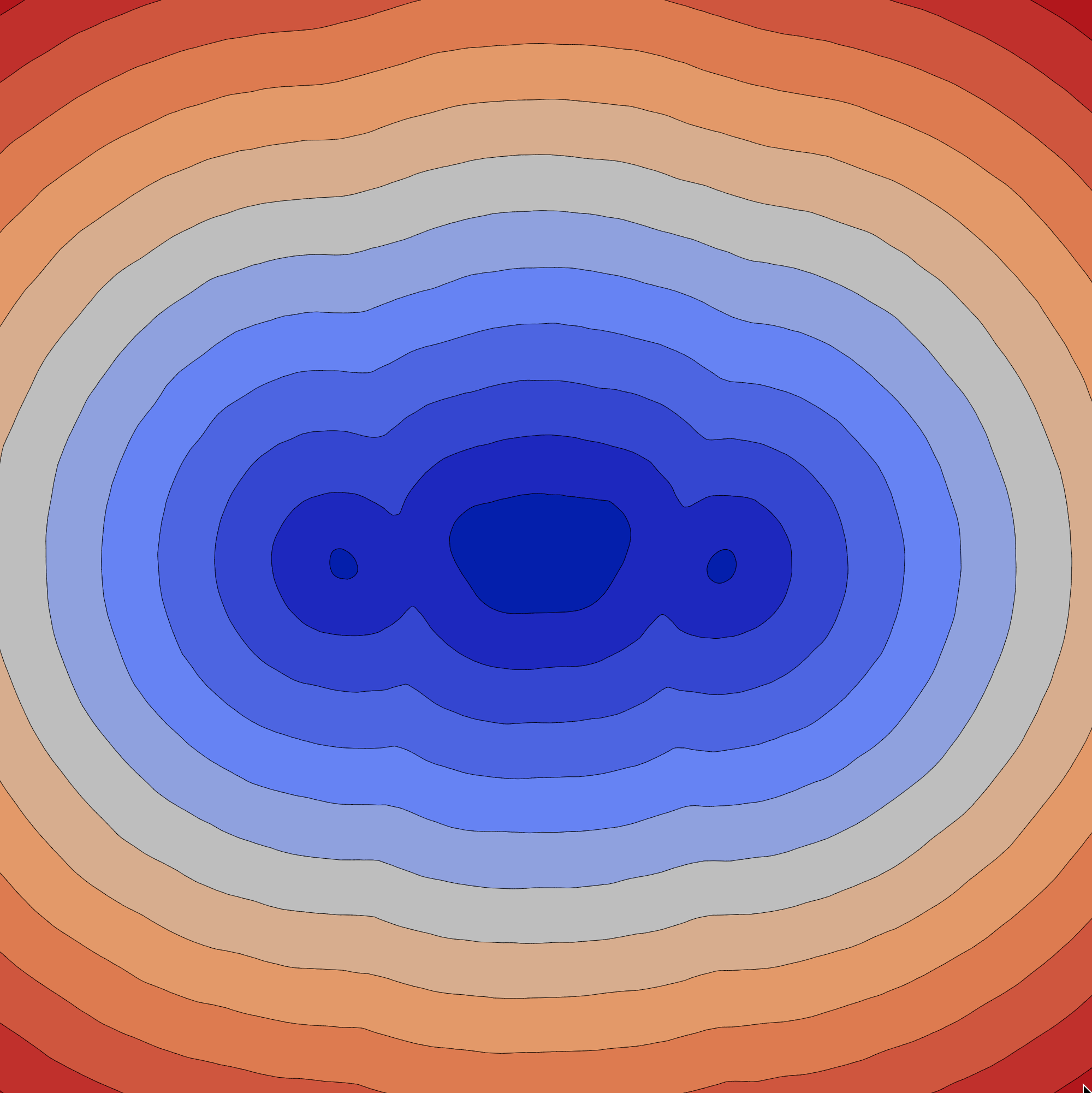}
		\caption{Level sets of a network $F$ from the experiment described in Section \ref{ss:curves_and_surfaces} shown along a cross-cut. Note how the iso-levels are equispaced, as encouraged by the loss in \Eqref{e:reconstruct}.}
		\label{fig:levelsets}
	\end{figure}
	
	\subsection{Level sets of reconstruction networks resemble signed distance function}
	Figure \ref{fig:levelsets} shows iso-levels of one of the networks from the experiment described in Section \ref{ss:curves_and_surfaces}. Note how the level sets resemble the level sets of a signed distance function.

	\section{Implementation details}
	% add "A" prefix (for appendix) to tables and figures
	\renewcommand{\thetable}{B\arabic{table}}
	\setcounter{table}{0}
	\renewcommand{\thefigure}{B\arabic{figure}}
	\setcounter{figure}{0}

	All experiments are conducted on a Tesla V100 Nvidia GPU using \textsc{pytorch} framework \citep{paszke2017automatic}.
	
	\begin{table}[h!]
		\centering
		\begin{tabular}{llllllrrr}
			\toprule
			ConvNet-2a & ConvNet-2b  & ConvNet-4a &  ConvNet-4b & FC-1 & FC-2 \\
			\midrule
			\textsc{Conv} 16 4x4+2  & \textsc{Conv} 32 5x5+1 &\textsc{Conv} 32 3x3+1    & \textsc{Conv} 32 3x3+1  & \textsc{FC} 512 & \textsc{FC} 509 \\
			\textsc{Conv} 32 4x4+2  & \textsc{MaxPool} 2x2   &\textsc{Conv} 32 3x3+1    & \textsc{Conv} 32 4x4+2  & \textsc{FC} 512 & \textsc{FC + skip} 509 \\
			\textsc{FC} 100         & \textsc{Conv} 64 5x5+1 & \textsc{MaxPool} 2x2    & \textsc{Conv} 64 3x3+1  & \textsc{FC} 512 & \textsc{FC + skip} 509 \\
			\textsc{FC} 10          & \textsc{MaxPool} 2x2   &\textsc{Conv} 64 3x3+1   & \textsc{Conv} 64 4x4+2  & \textsc{FC} 1   & \textsc{FC + skip} 509 \\
			& \textsc{FC} 512        &\textsc{Conv} 64 3x3+1    & \textsc{FC} 512         &                 & \textsc{FC + skip} 509 \\
			& \textsc{FC} 10         &\textsc{MaxPool} 2x2    & \textsc{FC} 512         &                 & \textsc{FC + skip} 509 \\
			&                        &\textsc{FC} 200    & \textsc{FC} 10          &                 & \textsc{FC + skip} 509 \\
			&                        &\textsc{FC} 200    &                         &                 & \textsc{FC} 1 \\
			&                        &\textsc{FC} 10    &                         &                 &                \\
			
			\bottomrule
		\end{tabular}
		\vspace{5pt}
		\caption{Our architectures. \textsc{Conv} $k w\times h+s$ corresponds to a convolution layer with $k$ channels, a kernel of size $w\times h$ and stride $s$. \textsc{FC} $n$ correspond to a fully connected layer with $n$ outputs. \textsc{FC + skip} indicates a skip connection to the input layer. Each \textsc{Conv}/\textsc{FC} layer is followed by a ReLU activation except for the last fully connected layer.}
		\label{tab:architectures}
	\end{table}
	
	\subsection{Parameters of experiments shown in Figure \ref{fig:teaser}}
	We train a $4$-layer MLP $F(x;\theta):\Real^2\times \Real^m \too \Real^2$, as in architecture FC-1, for $1000$ epochs using the \textsc{     Adam} optimizer \cite{kingma2014adam} with learning rate $0.001$. For the geometrical SVM loss we use $\lambda = 0.001$. Training set is composed of $16$ points in $\Real^2$, all of which lie inside $[0,0.5]^2$. Batch size is $1$. The sample network makes a maximum of $20$ iterations for the projection procedure.
	
	% Figure \ref{fig:teaser} illustrates a potential application of our method to classification tasks. We trained a binary classifier with four different loss functions: (a) Cross-Entropy loss (baseline). In (b-c) we use our sample network to directly control the decision boundary so that (i) the network correctly classifies the training set and (ii) the decision boundary passes at-least $\varepsilon$ away from the traiing set (we used $L^\infty$ and $L^2$, respectively)\footnote{More precisely, our loss here approximates $\max (0 , \varepsilon - \mathrm{sdist}(\gS(\theta),\text{train\, examples}) )$, where $\mathrm{sdist}$ is the signed $L^\infty$ or $L^2$ distance.}.  (d) We use our sample networks to generalize the soft margin linear svm to neural networks. By sampling the network $1$,$-1$ and $0$ level-sets (blue, red and black lines) we approximate the network output margin. In turn. we  couple the input margin to be proportional to the output margin, see \eqref{hing}. 
	
	\subsection{Classification generalization}\label{app:imp_gen}
	In Table~\ref{tab:generalization_hyper} we summarize all hyper-parameters used in the generalization experiments (Section~\ref{sec:exp_clasgen}). For cross-entropy and hinge losses we checked learning rates of $0.001, 0.005, 0.01, 0.02$ and chose the ones that performed best. All models were trained using SGD (Nesterov) optimizer with momentum $0.9$ and weight decay $10^{-4}$.
	\begin{table}[H]
		\centering
		\setlength\tabcolsep{3pt}
		\begin{adjustbox}{max width=\textwidth}
			\aboverulesep=0ex
			\belowrulesep=0ex
			\renewcommand{\arraystretch}{1.5}
			\begin{tabular}{|l|l|l|l|}
				\hline
				\backslashbox{Params}{Dataset} & \multicolumn{1}{c|}{MNIST} & \multicolumn{1}{c|}{Fashion-MNIST} & \multicolumn{1}{c|}{CIFAR10} \\ \toprule
				Architecture & ConvNet-2b & ConvNet-2b & ConvNet-4b \\ \hline
				\begin{tabular}[c]{@{}l@{}}Geometric \\ SVM $\lambda$\end{tabular} & \begin{tabular}[c]{@{}l@{}}$\lambda$ grows linearly from\\  $0.01$ to $0.2$ over $50$ epochs\end{tabular} & \begin{tabular}[c]{@{}l@{}}$\lambda$ grows linearly from \\ $0.01$ to $0.2$ over $50$ epochs\end{tabular} & \begin{tabular}[c]{@{}l@{}}$\lambda$ grows linearly from\\  $0$ to $0.01$ over $50$ epochs\end{tabular} \\ \hline
				$\#$ Epochs & $200$ & $200$ & $100$ \\ \hline
				Batch size & $256$ ($32$ for fraction $\leq0.3$) & $32$ & $32$ \\ \hline
				\begin{tabular}[c]{@{}l@{}}$\#$ Iterations for \\ projection proc.\end{tabular} & $20$ & $20$ & $20$ \\ \hline
				\multicolumn{1}{|l|}{\begin{tabular}[c]{@{}l@{}}Initial \\ learning rate\end{tabular}} & \multicolumn{1}{l|}{$0.02$} & \multicolumn{1}{l|}{$0.02$} & \multicolumn{1}{l|}{$0.01$} \\ \hline
				\begin{tabular}[c]{@{}l@{}}Learning rate\\ decay\end{tabular} & \begin{tabular}[c]{@{}l@{}}multipled by $0.5$ at epochs\\  $50, 100, 120, 140, 160, 180$\end{tabular} & \begin{tabular}[c]{@{}l@{}}multipled by $0.5$ at epochs\\  $50, 100, 120$\end{tabular} & \begin{tabular}[c]{@{}l@{}}multipled by $0.5$ at epoch\\  $50$\end{tabular} \\ \hline
				% Optimizer & SGD (momentum 0.9) & SGD (momentum 0.9) & SGD (momentum 0.9) \\ \hline
			\end{tabular}
		\end{adjustbox}     
		\vspace{3pt}
		\caption{Generalization experiments hyperparameters}
		\label{tab:generalization_hyper}
	\end{table}
	
	% \paragraph{MNIST.} All loss functions are optimized for $200$ epochs with batch size $256$ and learning rate set set to $0.02$ (for fraction $\leq0.3$ batch size is $32$).  For the geometrical SVM loss, $\lambda$ grows linearly from $0.01$ to $0.2$ over $40$ epochs. . Learning rate is decreased by half at epochs $50, 100, 120, 140, 160, 180$. 
	
	% \paragraph{FashionMNIST.} All loss functions are optimized for $200$ epochs with batch size $32$ and learning rate set to $0.02$. For the geometrical SVM loss, $\lambda$ grows linearly from $0.01$ to $0.2$ over $50$ epochs. For cross-entropy and hinge loss we checked learning rates of $0.01,0.001$ and batch sizes of $32, 128, 256$ (for fraction $\leq0.3$ batch size is $32$) and chose the best one ($lr=0.01; bs=256$ for both). Learning rate is decreased by half at epochs $50, 100, 120$.
	
	% \paragraph{CIFAR10.} All loss functions are optimized for $200$ epochs with batch size $32$  . For our trained model, $\lambda$ grows linearly from $0$ to $0.01$ over $50$ epochs, and learning rate $0.02$. For cross-entropy and hinge loss models we checked learning rates of $0.02,0.01,0.005,0.001$ and chose the best one ($0.01$ for cross entropy loss and $0.02$ for hinge loss). For all models the learning rate is decreased by half at epoch $50$. In 4 of 20 re-runs for fraction $0.2$ the training did not manage to converge, we omitted these results from the overall results. This phenomenon did not occur elsewhere.
	
	\subsection{Adversarial robustness} \label{app:adv_rob}
	
	We describe the parameters used in the experiments shown in Secion~\ref{sec:exp_adv}.
	
	\subsubsection*{Training parameters} 
	We use the networks described in Table~\ref{tab:architectures}, labeled ConvNet-4a and ConvNet-4b (following \cite{wong2017provable}) for the MNIST and CIFAR10 experiments respectively. Additionally, for CIFAR10 we add an experiment with ResNet-18 architecture as in \cite{zhang2019theoretically}.  All networks are trained with batches of size $128$. For the projection on the zero levelset procedure, we used the false-position method with a maximum of $40$ iterations per batch. The standard models are trained using cross-entropy loss for $200$ epochs on MNIST and CIFAR10 respectively (batch-size and learning rates are similar to the above mentioned models). All our models are trained using \textsc{Adam} optimizer \cite{kingma2014adam}.

	\paragraph{Bounded Attack (Table~\ref{tab:robustness_results})} We use the \textit{advertorch} library \cite{ding2018advertorch}. The attacks parameters are, for MNIST: $\varepsilon_\text{attack}=0.3$, PGD-iterations $40$ and $100$ and step size $0.01$. For CIFAR10: $\varepsilon_\text{attack}=8/255$, PGD-iterations $20$ and step size $0.003$. All models are evaluated at epoch $200$, except for Madry defense with ResNet-18 architecture evaluated at epoch $50$.
	
	\begin{table}[!htb]
		\centering
		\setlength\tabcolsep{2pt} % default value: 6pt
		\begin{subtable}{.45\linewidth}
			%   \centering
			\caption{Robust Accuracy Xent}
			\vspace{-5pt}
			\begin{tabular}{c}
				\begin{adjustbox}{max width=\textwidth}
					\aboverulesep=0ex
					\belowrulesep=0ex
					\renewcommand{\arraystretch}{1.1}
					\begin{tabular}[t]{l||c|c|c|c}
						\toprule
						\backslashbox{Target}{Source} & Standard & Madry & Trades & Ours \\
						\midrule
						Madry & 98.96\% & 96.04\% & 97.76\% & 99.21\% \\
						Trades & 98.57\% & 97.46\% & 96.78\% & 98.87\% \\
						Ours & 99.04\% & 97.78\% & 97.95\% & 99.23\% \\
						\bottomrule
					\end{tabular}
				\end{adjustbox}
			\end{tabular}
		\end{subtable} \qquad
		\begin{subtable}{.45\linewidth}
			%   \centering
			\caption{Robust Accuracy Margin}
			\vspace{-5pt}
			\begin{tabular}{c}
				\begin{adjustbox}{max width=\textwidth}
					\aboverulesep=0ex
					\belowrulesep=0ex
					\renewcommand{\arraystretch}{1.1}
					\begin{tabular}[t]{l||c|c|c|c}
						\toprule
						\backslashbox{Target}{Source} & Standard & Madry & Trades & Ours \\
						\midrule
						Madry & 98.95\% & 96.11\% & 97.81\% & 98.78\% \\
						Trades & 98.56\% & 97.5\% & 96.74\% & 98.46\% \\
						Ours & 99.04\% & 97.87\% & 97.99\% & 97.35\% \\
						\bottomrule
					\end{tabular}
				\end{adjustbox}
			\end{tabular}
		\end{subtable} 
		\caption{ MNIST: Comparison of our method and baseline methods under black-box PGD$^{40}$ attack with $\varepsilon_\text{attack} = 0.3$. Rows (target) are the attacked models. All models are trained with ConvNet-4a architecture. Diagonal represents white-box attacks.}
	\end{table}
	
	\begin{table}[!htb]
		\centering
		\setlength\tabcolsep{2pt} % default value: 6pt
		\begin{subtable}{.45\linewidth}
			%   \centering
			\caption{Robust Accuracy Xent}
			\vspace{-5pt}
			\begin{tabular}{c}
				\begin{adjustbox}{max width=\textwidth}
					\aboverulesep=0ex
					\belowrulesep=0ex
					\renewcommand{\arraystretch}{1.1}
					\begin{tabular}[t]{l||c|c|c|c}
						\toprule
						\backslashbox{Target}{Source} & Standard & Madry & Trades & Ours \\
						\midrule
						Madry & 61.5\% & 41.53\% & 49.76\% & 50.97\% \\
						Trades & 67.84\% & 54.72\% & 41.89\% & 53.11\% \\
						Ours & 68.43\% & 56.71\% & 54.47\% & 38.45\% \\
						\bottomrule
					\end{tabular}
				\end{adjustbox}
			\end{tabular}
		\end{subtable} \qquad
		\begin{subtable}{.45\linewidth}
			%   \centering
			\caption{Robust Accuracy Margin}
			\vspace{-5pt}
			\begin{tabular}{c}
				\begin{adjustbox}{max width=\textwidth}
					\aboverulesep=0ex
					\belowrulesep=0ex
					\renewcommand{\arraystretch}{1.1}
					\begin{tabular}[t]{l||c|c|c|c}
						\toprule
						\backslashbox{Target}{Source} & Standard & Madry & Trades & Ours \\
						\midrule
						Madry & 61.46\% & 39.13\% & 39.14\% & 51.08\% \\
						Trades & 67.58\% & 53.15\% & 38.25\% & 53.1\% \\
						Ours & 68.42\% & 55.85\% & 54.04\% & 38.54\% \\
						\bottomrule
					\end{tabular}
				\end{adjustbox}
			\end{tabular}
		\end{subtable} 
		\caption{ CIFAR10: Comparison of our method and baseline methods under black-box PGD$^{20}$ attack with $\varepsilon_\text{attack} = 0.031$. Rows (target) are the attacked models. All models are trained with ConvNet-4b architecture. Diagonal represents white-box attacks.}
	\end{table}

	\subsection{Surface Reconstruction}
	We describe the parameters used for the experiments in Section~\ref{ss:curves_and_surfaces}. For the Faust benchmark, the network architecture is set to FC-2 (similarly to \cite{chen2018learning,park2019deepsdf}) and is used both for our model and AtlasNet. The optimization is done using the \textsc{Adam} optimizer, batch size set to $10$ and the initial learning rate is set to $0.001$ (decreased by half at epochs $500$,$1500$,$3500$).
	Some additional implementation details are: first, we set the parameter $\lambda$ in our reconstruction loss to grow linearly from $1$ to $5$ over $1000$ epochs. Next, to generate samples of $\gS_t$ we add Gaussian noise ($\sigma = 0.1$) to the input batch, randomly sample half of the points and use it as initialization for the projection procedure to $\gS_0$. The other half is used to sample various level sets $\gS_t$ (see \Eqref{e:reconstruct}). The number of iterations for the projection procedure is set to 10.
	
	For the curve reconstruction experiment, the architecture used is FC-1 with the minor difference that the last layer output size is $2$. The ground truth is generated by randomly sampling $6$ points in space and generating a curve passing through the points, using cubic spline interpolation. We generate the input point cloud by sampling the ground truth curve and adding small Gaussian noise. The sample size is $300$ and sample points are chosen using Farthest Point Sampling. We generate samples from $\gS_0$ using the same procedure described above with the minor difference that the entire batch is used as initialization for the projection procedure (other, non-zero level sets are not sampled).

	\section{Proofs} \label{app:proofs}
	\setcounter{theorem}{0}
	\setcounter{lemma}{0}

	\begin{lemma}\label{lem:newton_sup}
		Let $\ell(x)=Ax+b$, $A\in \Real^{l\times d}$, $b\in \Real^l$, $\ell < d$, and $A$ is of full rank $l$. Then \Eqref{e:newton} applied to $F(x)=\ell(x)$ is an orthogonal projection on the zero level-set of $\ell$, namely, on $\set{x \ \vert \ \ell(x)=0}$.   
	\end{lemma}
	\begin{proof}
		Let $p\in\Real^d$ be the starting point. A single generalized Newton iteration (\Eqref{e:newton}) is
		\begin{equation}
		\label{e:newton_analysis}
		p^{\mathrm{next}}=p-A^\dagger (Ap+b).
		\end{equation} First, $p^{\mathrm{next}}$ is indeed on the level set because: $\ell(p^{\mathrm{next}})= A(p-A^\dagger (Ap+b))+b = 0$, where we used the fact that $AA^\dagger A = A$, and $AA^\dagger=I$ (since $\mathrm{rank}(A)=l$).
		Furthermore, from \Eqref{e:newton_analysis} we read that $p^{\mathrm{next}}-p \in \mathrm{Im} A^\dagger$ and therefore $p^{\mathrm{next}}-p \in \mathrm{Im} A^T$. This implies that $p^{\mathrm{next}}-p \perp \mathrm{Ker} A$. But $\mathrm{Ker} A$ is the tangent space of the level set $\set{x\,\vert \, \ell(x)=0}$, so $p^{\mathrm{next}}$ is the orthogonal projection of $p$ on the zero level set of $\ell$.
	\end{proof}
	
	\begin{lemma}\label{app:lemma_velocity}
		The columns of the solution in \Eqref{e:min_norm_sol}, namely $D_\theta p(\theta_0)$, are in the orthogonal space to the level set $\gS(\theta_0)$ at $p_0$. 
	\end{lemma}
	\begin{proof}
		$D_\theta p\in \Real^{d\times m}$ describes the speed of $p$ w.r.t.~each of the parameters in $\theta$. If we assume $A:=D_x F(p;\theta_0)$ is of full rank $l$, which is the generic case, then  the Moore-Penrose inverse has the form $A^\dagger = A^T (A A^T)^{-1}$. This indicates that the columns of $D_\theta p(\theta_0) = -A^\dagger D_\theta F(p;\theta_0)\in \Real^{d\times m}$ belong to $\mathrm{Im} A^T$, which in turn implies that they are orthogonal to $\mathrm{Ker} A$, which is the tangent space of the level set at the point $p_0$
	\end{proof}
	
	\begin{theorem}\label{e:levelset_universality_sup}
		Any watertight, not necessarily bounded, piecewise linear hypersurface $\gM\subset \Real^d$ can be exactly represented as the neural level set $\gS$ of a multilayer perceptron with ReLU activations, $F:\Real^{d}\too\Real$. 
	\end{theorem}
	\begin{proof}
		Let $h_i(x)=a_i^T x + b_i=0,\ i\in [k]$ denote the planes supporting the faces of $\gM$ where $a_i$ are chosen to be the outward normals to $\gM$. Since $\gM$ is watertight, it is the boundary of a $d$-dimensional polytope $P$.
		
		For each $\lambda \in \set{-1,0,1}^k$, let $P_\lambda = \cap_{i\in [k]} \set{x \ \vert \ \lambda_i h_i(x) \geq 0}$. Simply put, $P_\lambda$ is a polytope that is the intersection of closed half-spaces defined by the some of the hyperplanes $h_i$. Out of all the possible $P_\lambda$'s, we're only interested in those that are contained in $P$, so we define $\Lambda=\set{\lambda\ \vert \ P_\lambda \subseteq P}$.
		Now, we wish to show that every point in the interior of the large polytope necessarily also lies in the interior of some small polytope in our collection, i.e that $\cup_{\lambda \in \Lambda} \mathring{P_\lambda} = \mathring{P}$. So let $x \in \mathring{P}$. There are two cases:
		
		Case 1: $h_i(x) \ne 0\ \forall i \in [k].$ That is, $x$ does not lie exactly on a hyper-plane. We can take the following polytope $P_\lambda$ which contains $x$ in its interior: $\lambda_i=\mathrm{sign}(h_i(x))$. We note that $\lambda \in \set{-1,1}^k$, and we call such a polytope \textit{minimal}. We argue that the interior of a minimal polytope is either completely inside $P$ or completely outside it. This is true because otherwise the minimal polytope will contain two points that are on two different sides of some hyper-plane, which is inconsistent with $\lambda \in \set{-1,1}^k$. In our case, we know that $P_\lambda$ and $P$ both contain $x$ in their interior, so necessarily $P_\lambda \subseteq P$, which means that $\lambda \in \Lambda$.
		
		Case 2: $\exists \set{i_1, ..., i_l} \subseteq [k]$ s.t. $h_i(x)=0\ \forall i \in \set{i_1, ..., i_l}$, and $h_i \ne 0\ \forall i \in [k]\setminus\set{i_1,...,i_l}$. In this case there is no minimal polytope that contains $x$ in its interior, so let us consider all of the minimal polytopes which contain $x$ on their boundary. Let $P_\mu$ be such a minimal polytope. As previously stated, the interior of $P_\mu$ is either completely inside $P$ or completely outside it, but since $x$ is both on the boundary of $P_\mu$ and in the interior of $P$ then necessarily $P_\mu$ is completely inside $P$, i.e, $P_\mu \subset P$. We are interested in the union of all such minimal polytopes. Note that for such a minimal polytope $P_\mu$, necessarily $\mu_i = sign(h_i(x))\ \forall i \in [k]\setminus\set{i_1,...,i_l}$. For $i \in \set{i_1,...,i_l}$, $\mu_i$ may receive any value in $\set{1,-1}$. Thus, the union of all such minimal polytopes is $P_\lambda$ where: $$\lambda_i=\begin{cases}0 &,\ i \in \set{i_1,...,i_l} \\ \mathrm{sign}(h_i(x)) &,\ otherwise \end{cases}$$
		which clearly contains $x$ in its interior and is itself contained in $P$ (because it is the union of minimal polytopes which are contained in $P$), i.e $\lambda \in \Lambda$.
		
		%Let us define the smallest polytope $P_\lambda$ that contains $x$ in its interior: $$\lambda_i=\begin{cases}0 &,\ i \in \set{i_1,...,i_l} \\ \mathrm{sign}(h_i(x)) &,\ otherwise \end{cases}$$
		%Our goal is to show that $\lambda \in \Lambda$. First, we argue that every minimal polytope that is %contained in $P_\lambda$ is also contained in $P$: let $P_\mu$ be a minimal, non-empty polytope s.t. %$P_\mu \subset P_\lambda$. Then, as previously stated, the interior of $P_\mu$ is either completely %inside $P$ or completely outside it. Note that necessarily $\mu_i\lambda_i \geq 0\ \forall i \in %[k]\setminus\set{i_1,...,i_l}$ (otherwise $P\mu$, $P_\lambda$ would wholly reside on the two sides of %some hyper-plane, which is only possible if $P_\mu$ is either empty, or not contained in $P_\lambda$), %but we also have $h_i(x) = 0\ \forall i \in \set{i_1,...,i_l}$, which means that $\mu_i h_i(x) \geq 0\ %\forall i \in [k]$. This, in turn, means that $x \in P_\mu$ (Specifically, $x$ is on the boundary on %$P_\mu$). Since $x$ is both in $P_\mu$ and in the interior of $P$, then necessarily $P_\mu \subset P$. %Finally: $P_\lambda=\cup \set{P_\mu \subset P_\lambda | P_\mu\ minimal} \subseteq P$, i.e, $\lambda \in %\Lambda$.

		%
		%Figure \yl{XXX} depicts a 2 dimensional example of a polygonal shape, its supporting planes $\ell_i$, and $\Omega_\lambda$. 
		%
		
		We are now ready to define a function which will receive positive values on the interior of $P$, negative values outside of $P$, and will have $\gM$ as its levelset:
		$$f(x)=\max_{\lambda\in\Lambda}\min_{i\in [k]} \lambda_i h_i(x)$$
		
		$f$ is a piecewise linear function and can, therefore, according to Theorem 2.1 in \cite{arora2016understanding}, be encoded as an MLP with ReLU activations. The idea is to build $\max$ operators using linear layers and ReLU via $\max \set{a,b} = \frac{\sigma(a-b)}{2} + \frac{\sigma(b-a)}{2} + \frac{a+b}{2}$, where $\sigma(x)=\max(0,s)$ is the ReLU activation. Using this binary $\max$, one can recursively build the $\max$ of a vector. $\min$ is treated similarly.
		
	\end{proof}

\end{document}